\documentclass[11pt]{article}
\usepackage[usenames,dvipsnames]{xcolor}
\definecolor{Gred}{RGB}{219, 50, 54}
\definecolor{Ggreen}{RGB}{60, 186, 84}
\definecolor{Gblue}{RGB}{72, 133, 237}
\definecolor{Gyellow}{RGB}{247, 178, 16}
\definecolor{ToCgreen}{RGB}{0, 128, 0}
\definecolor{myGold}{RGB}{231,141,20}
\definecolor{myBlue}{rgb}{0.19,0.41,.65}
\definecolor{myPurple}{RGB}{175,0,124}
\usepackage{thm-restate}
\usepackage{hyperref}
\hypersetup{
      colorlinks=true,
  citecolor=ToCgreen,
  linkcolor=Sepia,
  filecolor=Gred,
  urlcolor=Gred
  }

\usepackage[margin=1in]{geometry}

\usepackage{microtype}
\usepackage{graphicx}
\usepackage{subfigure}
\usepackage{booktabs} % for professional tables

\usepackage{dsfont}

\renewcommand{\epsilon}{\varepsilon}

\usepackage{amsmath}
\usepackage{amssymb}
\usepackage{mathtools}
\mathtoolsset{showonlyrefs}
\usepackage{amsthm}

\theoremstyle{plain}
\newtheorem{theorem}{Theorem}[section]
\newtheorem{remark}{Remark}
\newtheorem{example}{Example}
\newtheorem{proposition}[theorem]{Proposition}
\newtheorem{lemma}[theorem]{Lemma}
\newtheorem{corollary}[theorem]{Corollary}
\theoremstyle{definition}

\newtheorem{assumption}{Assumption}
\usepackage[textsize=tiny]{todonotes}

\newcommand{\poly}{\mathrm{poly}}

\newcommand{\calN}{\mathcal{N}}

\newcommand{\wt}{\widetilde}

\newcommand{\Id}{\textup{Id}}
\newcommand{\Tr}{\mathop{\textup{Tr}}}
\newcommand{\R}{{\mathbb{R}}}

\renewcommand{\vec}[1]{\mathbf{#1}}
\newcommand{\wh}{\widehat}

\newcommand{\KL}[2]{\mathsf{KL}\left(#1\|#2\right)}

\renewcommand{\d}{{\mathrm{d}}}

\DeclarePairedDelimiter{\norm}{\lVert}{\rVert}
\DeclarePairedDelimiter{\iprod}{\langle}{\rangle}
\DeclarePairedDelimiter{\brk}{[}{]}
\DeclarePairedDelimiter{\brc}{\{}{\}}

\makeatletter
\def\Pr{\@ifnextchar[{\@witha}{\@withouta}}
\def\@witha[#1]{\mathop{\mathbb{P}}_{#1}\brk}
\def\@withouta{\mathop{\mathbb{P}}\brk}
\makeatother

\makeatletter
\def\E{\@ifnextchar[{\@withb}{\@withoutb}}
\def\@withb[#1]{\mathbb{E}_{#1}\brk}
\def\@withoutb{\mathbb{E}\brk}
\makeatother

\makeatletter
\def\Var{\@ifnextchar[{\@withc}{\@withoutc}}
\def\@withc[#1]{\mathop{\mathbb{V}}_{#1}\brk}
\def\@withoutc{\mathop{\mathbb{V}}\brk}
\makeatother

\makeatletter
\def\bone{\@ifnextchar[{\@withd}{\@withoutd}}
\def\@withd[#1]{\mathds{1}_{#1}\brk}
\def\@withoutd{\mathds{1}\brk}
\makeatother

\makeatletter
\def\psE{\@ifnextchar[{\@withe}{\@withoute}}
\def\@withe[#1]{\td{\mathop{\mathbb{E}}}_{#1}\brk}
\def\@withoute{\td{\mathop{\mathbb{E}}}\brk}
\makeatother

\DeclarePairedDelimiterX{\expectarg}[1]{[}{]}{%
  \ifnum\currentgrouptype=16 \else\begingroup\fi
  \activatebar#1
  \ifnum\currentgrouptype=16 \else\endgroup\fi
}

\newcommand{\innermid}{\nonscript\;\delimsize\vert\nonscript\;}
\newcommand{\activatebar}{%
  \begingroup\lccode`\~=`\|
  \lowercase{\endgroup\let~}\innermid 
  \mathcode`|=\string"8000
}

\newcommand{\xl}{x^\leftarrow}

\newcommand{\wtx}{\wt{x}}
\newcommand{\wtxl}{\wt{x}^\leftarrow}
\newcommand{\xllam}{x^{\leftarrow,\lambda}}
\newcommand{\wtxlam}{\wtx^\lambda}

\newcommand{\TV}{\mathsf{TV}}
\newcommand{\vecv}{v}
\renewcommand{\div}{{\rm div}}
\newcommand{\Lsc}[1]{L_{{\sf sc}, #1}}
\newcommand{\Lfx}{L_{f;{\sf x}}}
\newcommand{\Lft}{L_{f;{\sf t}}}
\newcommand{\Lg}{L_g}
\newcommand{\gmax}{g_{\rm max}}
\newcommand{\Lhigh}{L_{\sf high}}

\title{Restoration-Degradation Beyond Linear Diffusions: \\ A Non-Asymptotic Analysis For DDIM-Type Samplers}
\author{
    Sitan Chen\thanks{Email: \texttt{sitanc@berkeley.edu}} \\
    UC Berkeley
        \and 
    Giannis Daras\thanks{Email: \texttt{giannisdara@utexas.edu}}\\
    UT Austin
        \and
    Alexandros G. Dimakis\thanks{Email: \texttt{dimakis@austin.utexas.edu}}\\
    UT Austin
}

\begin{document}

\maketitle



\begin{abstract}
We develop a framework for non-asymptotic analysis of deterministic samplers used for diffusion generative modeling.
Several recent works have analyzed \emph{stochastic} samplers using tools like Girsanov's theorem and a chain rule variant of the interpolation argument. Unfortunately, these techniques give vacuous bounds when applied to deterministic samplers. 
We give a new operational interpretation for deterministic sampling by showing that one step along the probability flow ODE can be expressed as two steps: 1) a restoration step that runs gradient ascent on the conditional log-likelihood at some infinitesimally previous time, and 2) a degradation step that runs the forward process using noise pointing back towards the current iterate. This perspective allows us to extend denoising diffusion implicit models to general, non-linear forward processes. We then develop the first polynomial convergence bounds for these samplers under mild conditions on the data distribution.
%
%To the best of our knowledge, this constitutes the first non-asymptotic analysis of deterministic sampling with diffusion models, distinct from existing analyses based on tools like Girsanov's theorem which are only applicable to the stochastic setting.
\end{abstract}

\section{Introduction}
Diffusion models~\cite{sohl_thermodynamics, ddpm, ncsn} have emerged as a powerful framework for generative modeling. One of the core components is corrupting samples at different scales, slowly molding the data into noise. The corruption process, also known as the \emph{forward process}, can be fully described by the intermediate distributions, $\{ q_t\}_{t \in [0, T]}$, it defines. Diffusion models learn to revert the forward process by approximating the \textit{score} function, i.e. the gradient of the log-likelihood, of the intermediate distributions $q_t$. 
% The celebrated result of \cite{dsm} poses the problem of learning the score as a supervised restoration problem.

Once the score function has been learned, one can generate samples by running the reverse stochastic differential equation (SDE) associated with the forward process~\cite{anderson1982reverse, song2020score}. In practice however, one can only run a suitable discretization of the SDE, and due to the recursive nature of the sampling procedure, the discretization error from previous steps can accumulate, leading to sampling drift away from the true reverse process. Other sources of error come from the approximation error in estimating the score~\cite{sehwag2022generating, ddpm, nichol2021improved} and from the starting distribution. Controlling the propagation of errors in the reverse SDE has been studied in the recent works of \cite{blomrorak2020generative, debetal2021scorebased, deb2022manifold, liu2022let, leelutan2022generative, pid2022manifolds, leelutannew, chen2022improved, chen2022sampling}.

A second family of sampling methods is that of \emph{deterministic} samplers. As noted in~\cite{song2020score}, one can derive such samplers via a deterministic ODE process, the \emph{probability flow ODE}, that satisfies the same Fokker-Planck equation (and hence has the same marginals $\brc{q_t}$) as the reverse SDE. A different work, DDIM~\cite{ddim}, derives deterministic samplers by considering a non-Markovian diffusion process that leads to the same training objective, but a different reverse process. The two formulations turn out to be equivalent up to a reparametrization~\cite{song2020score, karras2022elucidating}. DDIM samplers can be interpreted as iterating a combination of two steps: a restoration step that recovers some rough final reconstruction of the current iterate at time $t$, and a degradation step that corrupts this rough estimate to time $t + h$. This interpretation can be extended to accommodate general \emph{linear} corruption processes~\cite{zhang2022gddim,soft_diffusion, bansal2022cold, zhang2022gddim}.

Deterministic samplers offer a number of advantages over stochastic ones. While the latter are typically state-of-the-art for image generation, they require a large \emph{number of function evaluations} which renders them impractical for many applications. The gap between sample quality for deterministic and stochastic samplers has been significantly narrowed in the recent work of~\cite{karras2022elucidating}. Deterministic samplers are typically much faster~\cite{ddim, nichol2021improved} and also useful for computing likelihoods~\cite{ddpm, song2020score}. Further, one of the most successful techniques for accelerating diffusion models, Progressive Distillation~\cite{salimans2022progressive}, requires deterministic samplers. Deterministic samplers also allow the exploration of the semantic latent space of the trained network~\cite{kwon2022diffusion}.

Despite their significance, there is currently limited theoretical understanding for deterministic samplers. Specifically, there is no analysis for their non-asymptotic convergence behavior, in contrast to stochastic samplers. Obtaining such an analysis is challenging because Girsanov's theorem\--- the main tool for bounding the propagation of errors when implementing the reverse SDE\--- and related techniques all yield vacuous bounds for deterministic samplers (see Section~\ref{sec:related}). 

Our contributions are twofold. We first propose a new operational interpretation for the reverse ODE that generalizes DDIM sampling to arbitrary, \emph{non-linear} forward processes. 

\begin{theorem}[Informal, see Section~\ref{sec:interpret}]\label{thm:interpret_informal}
    Denote by $h$ the infinitesimally small step size with which we discretize the probability flow ODE. Let $\ell\in\mathbb{N}$ be a parameter for which $\ell \to \infty$ and $\ell h \to 0$. For any forward process, running the probability flow ODE for time $h$ is equivalent to running the following two steps: 1) restoring the current iterate to $\ell h$ time steps in the past via a step of gradient ascent on conditional log-likelihood, 2) degrading this by $(\ell - 1)h$ steps by simulating the forward process with noise pointing in the direction of the current iterate.
\end{theorem}

\noindent We then complement this new asymptotic result with a non-asymptotic proof that the sampler from this operational interpretation converges to the true process. This yields a deterministic sampling analogue of recent non-asymptotic analyses of stochastic samplers for diffusion models~\cite{chen2022sampling,leelutannew,chen2022improved}:

\begin{theorem}[Informal, see Theorem~\ref{thm:main}]
    Under mild assumptions on the smoothness of the data distribution (in particular, the distribution can be arbitrarily non-log-concave), the deterministic sampler arising from Theorem~\ref{thm:interpret_informal} generates samples for which the KL divergence with respect to the data distribution is small provided $\ell h$ and $\ell^{-1}$ are polynomially small in the dimension and other problem-specific parameters.
\end{theorem}

\noindent As a corollary, our techniques imply that the same bounds hold for the Euler discretization of the probability flow ODE, yielding, to our knowledge, the first non-asymptotic analysis of this sampler.

\paragraph{Roadmap.} In Section~\ref{sec:prelims} we briefly recall the forward and reverse processes in diffusion generative modeling. In Section~\ref{sec:interpret}, we give our new operational interpretation for the probability flow ODE for general, non-linear diffusions. In Section~\ref{sec:discrete_preapp} we formally state our main non-asymptotic guarantee, Theorem~\ref{thm:main}, and give an overview of the proof, deferring the technical details to the Appendix.

In Appendix~\ref{app:learning_rate} we motivate the choice of certain learning rate parameter that arises in Section~\ref{sec:interpret}. In Appendix~\ref{app:proof_prelims} we provide preliminary calculations for the proof of Theorem~\ref{thm:main}. In Appendix~\ref{sec:generic}, we give a generic bound on the distance between two processes driven by ODEs with similar drifts, one of which is an interpolation of a discrete-time process. Finally, in Appendix~\ref{sec:discrete_analysis} we apply this generic bound to our setting, bound the difference in drifts between the probability flow ODE and our sampler, and prove Theorem~\ref{thm:main}.

\section{Preliminaries}
\label{sec:prelims}
In this work we consider a general forward process driven by a stochastic differential equation (SDE) of the form:
\begin{equation}
    \d x_t = f_t(x_t)\, \d t + g(t)\, \d W_t, \qquad x_0 \sim q\,, \label{eq:forward}
\end{equation}
where $(W_t)$ is a standard Brownian motion in $\R^d$. Let $q_t$ denote the law of $x_t$, so that $q_0 = q$.

Suppose we run the forward process up to a terminal time $T > 0$. Under mild conditions on the diffusion (see e.g. \cite{anderson1982reverse,follmer1985reversal,catetal2022timereversal}) which are satisfied by the processes we consider in this work, there is a suitable reverse process given by an SDE such that the marginal distribution at time $t$ is given by $q_{T-t}$. For convenience, we will often refer to $q_{T-t}$ as $q^\leftarrow_t$.

In fact, there is an entire family of SDEs with this property. For any $\lambda \ge 0$, consider the process $(\xllam_t)_{0 \le t \le T}$ given by
\begin{equation}
    \d \xllam_t = -\bigl\{f_{T-t}(\xllam_t) - \frac{1 + \lambda^2}{2}g(T-t)^2\nabla \ln q^\leftarrow_t(\xllam_t)\bigr\}\, \d t + \lambda g(T-t)\d W_t\ ,\qquad \xllam_0 \sim q^\leftarrow_0 \,.
    \label{eq:general_stochastic}
\end{equation}
By checking the Fokker-Planck equation, one sees that the marginal distribution of $\xllam_t$ is indeed given by $q^\leftarrow_t$.

% here is states that under mild conditions on $f, g$, this process is reversible, and the diffusion process $(\xl_t)_{0 \le t \le T}$ given by the SDE
% \begin{align}
%     \d \xl_t &= -\brc{f_{T - t}(\xl_t) - g^2(T - t)\nabla \ln q^\leftarrow_t(\xl_t)}\d t \\ &\quad\quad + g(T-t)\d W^{\leftarrow}_t\,\label{eq:backward_sde}
% \end{align}
% with $\xl_0 \sim q_T$ serves as a suitable reverse process. The key feature of this reverse process is that for any $t\in[0,T]$, the marginal distribution on $\xl_t$ is given by $q_{T-t}$. 

One notable process in this family corresponds to the case of $\lambda = 0$. This is a \textit{deterministic} process, denoted $(\xl_t)_{0 \le t \le T}$, driven by the probability flow ODE~\cite{song2020score}.
\begin{equation}
	\d\xl_t = - \brc{f_{T-t}(\xl_t) - \frac{1}{2}g(T-t)^2 \nabla \ln q^\leftarrow_t(\xl_t)}\,\d t \ ,\label{eq:basicode}
\end{equation}
with $\xl_0 \sim q^\leftarrow_0$. % so that the marginal distribution of the reverse process at time $T - t$ agrees with the marginal distribution of the forward process at time $t$ for any $0 \le t \le T$

% More generally, there is an entire family of SDEs that can serve as the reverse of the forward process defined in \eqref{eq:forward}~\cite{zhang2022fast}. Indeed, for any $\lambda > 0$, we can consider
% \begin{gather}
%     \d \xllam_t = -\bigg(f_{T-t}(\xllam_t) - \nonumber \\ \frac{1 + \lambda^2}{2}g(T-t)^2\nabla \ln q^\leftarrow_t(\xllam_t)\bigg)\, \d t + \nonumber \\ + \lambda g(T-t)\d W_t\ ,\qquad \xllam_0 \sim q_T\,.
%     \label{eq:general_stochastic}
% \end{gather}

% Note that the probability flow ODE is merely one of the endpoints of this family, corresponding to the choice of $\lambda = 0$.

In the diffusion model literature, there are two popular choices of forward process: the \emph{variance exploding (VE) SDE}~\cite{song2020score, ncsn, ncsnv2}, which corresponds to $f_t(x_t) = 0$, $g(t) = \sqrt{\frac{\d \sigma_t^2}{\d t}}$ for some increasing function $\sigma^2_t$; and the \emph{variance preserving (VP) SDE}~\cite{ddpm}, which corresponds to $f_t(x_t) = -\frac{1}{2}\beta_t x_t, \ g(t) = \sqrt{\beta_t}$ for some variance schedule $\beta_t$. These two choices are used in state-of-the-art diffusion models~\cite{dhariwal2021diffusion, ddpm++} and form the backbone of systems like DALL$\cdot$E 2~\cite{dalle2}, Imagen~\cite{imagen}, and Stable Diffusion~\cite{latent_diffusion}.

\section{Operational Interpretation for the Probability Flow ODE}
\label{sec:interpret}
\subsection{Warmup: linear SDEs and DDIM}

We begin by recalling the interpretation of the probability flow ODE associated to the variance exploding (VE)~\cite{song2020score} SDE as a \emph{denoising diffusion implicit model} (DDIM)~\cite{ddim}. For simplicity of exposition, we specialize to the case of $\sigma^2_t = t$, which corresponds to the forward process
\begin{equation*}
    \d x_t = \d W_t, \qquad x_0 \sim q\,.
\end{equation*}
According to \eqref{eq:basicode}, the associated probability flow ODE is:
\begin{equation}
    \d \xl_t = \frac{1}{2}\nabla \ln q^\leftarrow_t(\xl_t)\,\d t, \qquad \xl_0 \sim q_T, \label{eq:vesde_ode}
\end{equation}
so that the marginal distribution of $\xl_t$ is $q^\leftarrow_t$ for any $0 \le t\le T$.
The perspective of DDIM offers an interesting operational interpretation of \eqref{eq:vesde_ode}. Fix some infinitesimally small step size $h$, and consider the following procedure for forming $\xl_{t+h}$ given $\xl_t$. We first produce an estimate for the \emph{beginning} $x_0$ of the forward process. 
% Specifically, we consider the first-order discretization of the forward process \sitan{i think the following equation is exact, not just a discretization?}:
Note that
\begin{equation}
    \xl_t = x_{T - t} = x_0 + \varepsilon\,\sqrt{T - t} \label{eq:x0xt}
\end{equation} 
for $\varepsilon\sim\calN(0,\Id)$, so by Tweedie's formula~\cite{efron2011tweedie}, the mean of the posterior distribution over $x_0$ given $\xl_t$, i.e. $\mathbb E[x_0|\xl_t]$, is exactly:
\begin{equation*}
    z\triangleq \mathbb{E}[x_0|\xl_t] = \xl_t + (T-t)\,\nabla \ln q^\leftarrow_t(\xl_t)\,.
\end{equation*}
Starting from $z$ and degrading it along the forward process from time $0$ to time $T - t$, we would end up with $z + \gamma\sqrt{T - t}$ for some Gaussian noise $\gamma\sim\calN(0,\Id)$. 

Here is the key idea behind DDIMs: suppose we instead took $\gamma$ to be the solution to
\begin{equation*}
    \xl_t = z + \gamma\sqrt{T - t}\,,
\end{equation*}
i.e. suppose we took $\gamma$ to be the ``simulated noise'' that would be needed to degrade $z$ into $\xl_t$, rather than fresh Gaussian noise. 

Now imagine running the forward process to degrade $z$ from time $0$ to time $T - (t + h)$, but using this simulated noise $\gamma = \frac{\xl_t - z}{\sqrt{T - t}}$ instead of Gaussian noise. It turns out that the resulting vector, which we will define $\xl_{t+h}$ to be, is approximately what we would get by running the probability flow ODE for time $h$ starting at $\xl_t$!

Indeed, the result of degrading $z$ in this fashion is
\begin{align}
    \xl_{t+h} &\triangleq z + \sqrt{T - (t+h)} \cdot \frac{\xl_t - z}{\sqrt{T - t}} \\
    &= \xl_t + (T-t)\cdot \Bigl(1 - \sqrt{1 - \frac{h}{T-t}}\Bigr)\cdot\nabla \ln q_{T-t}(\xl_t)\,.
\end{align}
Observe that as $h\to 0$, the iterate $\xl_{t+h}$ tends to $\xl_t + \frac{h}{2}\nabla \ln q_{T-t}(\xl_t)$. Therefore, the above interpretation indeed recovers the probability flow ODE \eqref{eq:vesde_ode} as claimed. The above generalizes without much difficulty to any linear diffusion~\cite{soft_diffusion, bansal2022cold}.

\subsection{General diffusions}
\label{sec:general_diffusions} 

Let us now consider the setting where the forward process is given by an arbitrary, possibly non-linear diffusion as in Eq.~\eqref{eq:forward} in Section~\ref{sec:prelims}, so that the associated probability flow ODE is given by Eq.~\eqref{eq:basicode}.
Unfortunately, as soon as we step away from the linear setting, the operational interpretation from the previous section breaks down. The key issue is that when forming our estimate $z$ for the beginning of the forward process, there is no longer any simple expression for the posterior mean conditioned on $\xl_t$. 

\paragraph{Restoration operator.} To get around this issue, our first insight is: instead of deriving an estimate for the beginning of the forward process, we instead derive one for the process \emph{$\ell h$ units of time in the past}, i.e. at time $T - t - \ell h$ of the forward process. In the previous section, we implicitly took $\ell = (T - t) / h$, but now $\ell$ is a parameter that needs to be tuned. Crucially, selecting $\ell$ such that $\ell h\to 0$ allows us to linearize around $T - t$.  In analogy with \eqref{eq:x0xt}, we get the approximate relation
\begin{align}
    \xl_t = x_{T - t} &\approx x_{T - t - \ell h} + \ell h\, f_{T - t - \ell h}(x_{T-t-\ell h})  + g(T - t - \ell h)\sqrt{\ell h}\cdot \varepsilon \\
    &\approx x_{T - t - \ell h} + \ell h\, f_{T - t}(\xl_t) + g(T - t) \sqrt{\ell h}\cdot \varepsilon
\end{align}
for $\varepsilon\sim \calN(0,\Id)$, where the approximations hold up to $o(h)$ additive error. Rearranging, we see that $x_{T - t - \ell h}$ is simply $\xl_t - \ell h f_{T - t}(\xl_t)$ plus some Gaussian noise of variance $\ell h g(T - t)^2$. So, again by Tweedie's formula, we find that the mean of the posterior distribution over $x_{T - t - \ell h}$ given $\xl_t$ is approximately
\begin{equation}
    z\triangleq \xl_t - \ell h\,\brc{ f_{T-t}(\xl_t) - g(T - t)^2\nabla \ln q^\leftarrow_t(\xl_t)}\,.
    \label{eq:tweedies_update}
\end{equation}
Borrowing terminology from \cite{bansal2022cold}, we refer to the map from $\xl_t$ to $z$ as the \emph{restoration operator}. Formally, for $t > s > 0$, define the restoration operator $R_{t\to s}(\cdot)$ by
\begin{equation}
    R_{t\to s}(x) \triangleq x - (t - s)f_t(x) + (t - s) g(t)^2 \nabla \ln q_t(x) \label{eq:restore}
\end{equation}
so that $z = R_{T - t \to T - t - \ell h}(\xl_t)$.
% This has the interpretation of taking an input $x$ at time $t$ of the forward process and running a single step of gradient ascent to maximize the conditional log-likelihood $\ln q_t(x_t\mid x_s = x)$.

% \TODO{discuss connection to drift term in reverse SDE and to Tweedie's formula / Gaussian integration by parts}

\paragraph{Restoration operator as gradient ascent.} Here we briefly remark that there turns out to be a different way of thinking about the restoration operator, namely as one step of gradient ascent.

Formally, given times $0 < t < s$, consider maximizing the conditional log-likelihood $\ln q^\leftarrow_s(\cdot \mid x^\leftarrow_t)$. This is equivalent to maximizing
\begin{equation}
	\ell_{\xl_t}(x) \triangleq \ln q^\leftarrow_t(\xl_t \mid \xl_s = x) + \ln q^\leftarrow_s(x). \label{eq:bayes}
\end{equation}
For $s$ which is infinitesimally larger than $t$, the law of $\xl_t$ conditioned on $\xl_s = x$ is Gaussian with mean and covariance approximately $x + f_{T - s}(x)\, (t-s)$ and $g(T - t)^2 (t-s) \, \Id$. We can thus compute the gradient of \eqref{eq:bayes} to get
\begin{equation}
    \nabla \ell_{\xl_t}(x) \approx \frac{1}{g(T - t)^2 (t-s)}\bigl(\Id + (t-s)\,\nabla f_{T - s}(x)\bigr)  \cdot  \bigl(\xl_t - x - f_{T - s}(x)\, (t-s)\bigr) +\nabla \ln q^\leftarrow_s(x)\,. \label{eq:gradient}
\end{equation}

Now consider taking a single gradient step with learning rate $\eta$ starting from $\xl_t$ to get $\xl_t + \eta\nabla \ell_{\xl_t}(\xl_t)$. In Appendix~\ref{app:learning_rate}, we show that in the special case where $q$ is Gaussian and the forward process is Ornstein-Uhlenbeck, the correct choice of learning rate to maximize the conditional log-likelihood with just one step of gradient ascent is
\begin{equation}
    \eta\triangleq 2 g(T-t)^2\cdot (t - s)\,. \label{eq:lr}
\end{equation}
In this case, note that
\begin{align}
    \xl_t + \eta\nabla\ell_{\xl_t}(\xl_t) &\approx \xl_t - (t - s)f_{T - t}(\xl_t) - (t - s)^2 (\nabla f_{T - s}(\xl_t))f_{T - s}(\xl_t) \\
    &\qquad\qquad{} + (t - s) g(T - t)^2 \nabla \ln q^\leftarrow_s(\xl_t)\\
    &\approx \xl_t - (t - s)f_{T - t}(\xl_t) +(t - s) g(T - t)^2 \nabla \ln q^\leftarrow_t(\xl_t),
    \label{eq:gd_update}
\end{align}
where in the second step we have dropped the second order term $(t - s)^2(\nabla f_{T - s}(\xl_t))f_{T - s}(\xl_t)$ and approximated $(t - s) g(T - t)^2\nabla \ln q^\leftarrow_s(\xl_t)$ to first order by $(t - s)g(T - t)^2\nabla \ln q^\leftarrow_t(\xl_t)$. Observe now that for $s=t - \ell h$, the update rule of \eqref{eq:gd_update} is the same as the update rule of \eqref{eq:tweedies_update}.

\paragraph{Degradation operator.} The remainder of the derivation proceeds along similar lines to the previous section. Given noise vector $\gamma\in\R^d$, define the \emph{degradation operator} $D^\gamma_{s,t}(\cdot)$ by
\begin{equation}
	D^\gamma_{s\to t}(x) \triangleq x + f_s(x)(t - s) + g(s)\sqrt{t - s} \cdot \gamma\,.\label{eq:degrade}
\end{equation}
This operator simply runs an Euler-Maruyama discretization of the forward process, starting at time $s$, for time $t - s$, with the noise taken to be $\gamma$.

Starting from $z$ and degrading it along the forward process from $T - t - \ell h$ to time $T - t$, we would end up with $D^\gamma_{T-t-\ell h \to T - t}(z) = z + \ell h\, f_{T-t - \ell h}(z) + g(T - t - \ell h)\sqrt{\ell h}\cdot \gamma$ for some Gaussian noise $\gamma\sim\calN(0,\Id)$. As before, we instead take $\gamma$ to be the simulated noise needed to degrade $z$ into $\xl_t$, which in this case is given by the solution to
\begin{equation}
    \xl_t = z + \ell h\, f_{T-t-\ell h}(z) + g(T - t - \ell h)\sqrt{\ell h}\cdot \gamma\,.
\end{equation}
To produce the next iterate $\xl_{t+h}$ of the reverse process, we use $\gamma$ to degrade $z$ from time $T - t - \ell h$ to $T - t - h$. The result is given by
\begin{align}
    \xl_{t+h} &= z + (\ell - 1) h \, f_{T-t-\ell h}(z) + g(T - t - \ell h)\sqrt{(\ell - 1) h}\cdot \frac{\xl_t - z - \ell h \, f_{T-t-\ell h}(z)}{g(T - t - \ell h)\sqrt{\ell h}}  \\
    &\approx z + (\ell - 1)h f_{T-t}(\xl_t) + \sqrt{1 - 1/\ell}\cdot \ell h g(T - t)^2 \nabla \ln q_{T-t}(\xl_t) \label{eq:firstapprox_overview}\\ 
    &= \xl_t - h f_{T-t}(\xl_t) + \ell h\cdot \left(1 - \sqrt{1 - 1/\ell}\right) \cdot g(T - t)^2\nabla \ln q_{T-t}(\xl_t) \,.\label{eq:update_rule}
\end{align}
where in the second step we approximated $f_{T - t - \ell h}(z)$ by $f_{T - t}(\xl_t)$ and dropped $o(\ell h)$ terms as we are assuming $\ell h \to 0$. Finally as $\ell \to \infty$, the right-hand side converges to $\xl_t - h\, \brc{f_{T-t}(\xl_t) - \frac{1}{2}g(T - t)^2\nabla \ln q_{T-t}(\xl_t)}$, which recovers the Euler discretization of the probability flow ODE. We note that this final step is the only place that requires taking $\ell \to \infty$. Finally, as we take $h \to 0$, the above recovers the probability flow ODE~\eqref{eq:basicode}.

We give a more formal description of the sampling algorithm that this operational interpretation suggests at the beginning of Section~\ref{sec:discrete_preapp}, where we give the main findings of our non-asymptotic analysis of this sampler.

\subsection{Extensions to other samplers}
The operational interpretation that we developed to extend DDIM to non-linear forward processes can be adapted in a relatively straightforward way to describe more general samplers. For example, in Equation \eqref{eq:general_stochastic}, we defined a more general family of reverse processes, each of which has the correct marginal law at time $t$. These can easily be described by a similar operational interpretation.

Specifically, consider the following process $(\wtx^\lambda_{kh})_{k\in\{0,\ldots,T/h\}}$. Given iterate $\wtx^\lambda_{kh}$, the preceding iterate $\wtx^\lambda_{(k-1)h}$ is defined as follows:
\begin{align}
    \wtxlam_{(k-1)h} = D^{\gamma'}_{(k-\ell)h \to (k-1)h}(z), \ \ \ \text{for} \ &z \triangleq R_{kh\to(k-\ell)h}(\wtxlam_{kh})\,, \label{eq:zlam_def} \\
    &\gamma' = \sqrt{1 - \frac{ \lambda^2}{\ell - 1}} \, \gamma + \frac{1}{\sqrt{\ell-1}}\lambda \nu, \label{eq:gamprime_def}\\
    &\gamma: D^\gamma_{(k-\ell)h\to kh}(z) = \wtxlam_{kh}, \qquad \nu \sim \calN(0,\Id)\,.
\end{align}
Note that we use the same restoration operator as before to arrive to $z$, which we then use to estimate the noise $\gamma$. The critical change to the framework is that now, to corrupt from $z$ to $\xl_{t + h}$, instead of just using the estimated noise, we use a linear combination of the estimated noise, $\gamma$ and \emph{fresh noise} $\nu$. 
% Specifically, in the degradation operator, we use a vector $\gamma'$, defined as follows:
% % Specifically, consider the following process $(\wtx^\lambda_{kh})_{k\in\{0,\ldots,\lfloor T/h\rfloor\}}$. Given iterate $\wtx^\lambda_{kh}$, the preceding iterate $\wtx^\lambda_{(k-1)h}$ is defined as follows:
% \begin{equation*}
%     \gamma' = \sqrt{1 - \frac{ \lambda^2}{\ell - 1}} \, \gamma + \frac{1}{\sqrt{\ell-1}}\lambda \nu, \quad \nu \sim \calN(0,\Id)\,. \label{eq:gamprime_def}
% \end{equation*}

The parameter $\lambda$ here controls how close the update rule is to the deterministic sampler. Trivially, for $\lambda=0$, we have a fully deterministic sampler, as before. For $\lambda=1$, the sampler becomes the reverse SDE sampler of~\cite{song2020score}.
The coefficients have been chosen such that if $\gamma$ were actually a draw from $\calN(0,\Id)$ instead of simulated noise, then $\gamma'$ would likewise be a draw from $\calN(0,\Id)$.

Note that
\begin{align}
    \wtxlam_{(k-1)h} &= z + (\ell - 1)h\, f_{(k - \ell)h}(z) + g((k - \ell)h)\sqrt{(\ell - 1)h}\cdot \gamma'  \\
    &\approx z + (\ell - 1)h\, f_{kh}(\wtxlam_{kh}) + g(kh)\sqrt{(\ell-1)h}\cdot \gamma' \label{eq:wtxlam_next}
\end{align}
where in the second step we approximated $f_{(k-\ell)h}(z)$ and $g((k-\ell)h)$ by $f_{kh}(\wtxlam_{kh})$ and $g(kh)$, dropping $o(h)$ terms.
By \eqref{eq:gamprime_def} and the definition of the estimated noise $\gamma$ in \eqref{eq:gamdef}, we have
\begin{equation}
    \gamma' = \sqrt{1 - \frac{ \lambda^2}{\ell - 1}} \cdot \frac{\wtxlam_{kh} - z - \ell h\, f_{(k-\ell)h}(z)}{g((k - \ell)h)\sqrt{\ell h}} + \frac{1}{\sqrt{\ell-1}}\lambda \nu \approx \sqrt{1 - \frac{ \lambda^2}{\ell - 1}} \cdot \frac{\wtxlam_{kh} - z - \ell h\, f_{kh}(z)}{g(kh)\sqrt{\ell h}} + \frac{1}{\sqrt{\ell-1}}\lambda \nu\,,
\end{equation}
where we approximated $f_{(k-\ell)h}(z)$ and $g((k-\ell)h)$ by $f_{kh}(z)$ by $g(kh)$. Substituting this into \eqref{eq:wtxlam_next} and recalling the definition of $z$ in \eqref{eq:zlam_def}, we have
\begin{align}
    \wtxlam_{(k-1)h} 
    % &\approx z + (\ell - 1)h\, f_{kh}(\wtxlam_{kh}) + g(kh)\sqrt{(\ell-1)h}\cdot \Bigl( \sqrt{1 - \frac{ \lambda^2}{\ell - 1}} \cdot \frac{\wtxlam_{kh} - z - \ell h\, f_{kh}(z)}{g(kh)\sqrt{\ell h}} + \frac{1}{\sqrt{\ell-1}}\lambda \nu\Bigr) \\
    &\approx \wtxlam_{kh} - h \, f_{kh}(\wtxlam_{kh}) + \ell h\, g(kh)^2 \nabla \ln q_{kh}(\wtxlam_{kh})
    \\
    &\qquad\qquad\qquad- \sqrt{1 - \frac{1}{\ell}}\cdot \sqrt{1 - \frac{ \lambda^2}{\ell - 1}} \cdot \ell h \, g(kh)^2\nabla \ln q_{kh}(\wtxlam_{kh}) + \lambda\sqrt{h}\, g(kh)^2\, \nu \\
    &{\scriptscriptstyle (\ell\to \infty)} = \wtxlam_{kh} - h\,\brc{f_{kh}(\wtxlam) - \frac{1 + \lambda^2}{2}g(kh)^2 \nabla \ln q_{kh}(\wtxlam)} + \lambda \sqrt{h}\, g(kh)^2\nu\,,
\end{align}
and in the last step we used that
\begin{equation*}
    \lim_{\ell\to\infty} \ell\Bigl(1 - \sqrt{1 - \frac{1}{\ell}}\cdot \sqrt{1 - \frac{\lambda^2}{\ell - 1}}\Bigr) = \frac{1 + \lambda^2}{2}\,.
\end{equation*}

\section{Discretization Analysis}
\label{sec:discrete_preapp}

In what follows, we provide a non-asymptotic convergence analysis for DDIM-type samplers as captured by Eq.~\eqref{eq:update_rule}. In Section~\ref{sec:define_sampler} we formally define the sampler in question. Then in Section~\ref{sec:results} we state our main results, which to the best of our knowledge constitute the first convergence analysis for deterministic sampling with diffusion models. In Section~\ref{sec:overview} we give a proof overview, deferring most of the technical details to Appendices~\ref{app:proof_prelims}, \ref{sec:generic}, and \ref{sec:discrete_analysis}.

\subsection{DDIM-type sampler}
\label{sec:define_sampler}

Motivated by the discussion in Section~\ref{sec:general_diffusions}, our analysis will focus on the process $(\wtx_{kh})_{k\in\{0,\ldots,T/h\}}$ defined backwards in time as follows. The iterate $\wtx_T$ is sampled from $q^\leftarrow_0$. Given iterate $\wtx_{kh}$, the preceding iterate $\wtx_{(k-1)h}$ is defined as follows:
% \begin{equation}
% 	\wtx_{(k-1)h} = \wtx_{kh} + D^\gamma_{(k-\ell)h\to (k-1)h}(z) - \sqrt{1 - 1/\ell}\,  D^\gamma_{(k-\ell)h \to kh}(z), \ \ \ z \triangleq R^{\eta_k}_{kh \to (k-\ell)h}(\wtx_{kh})\,. \label{eq:mle}
% \end{equation}
\begin{equation}
    \wtx_{(k-1)h} = D^\gamma_{(k-\ell)h \to (k-1)h}(z), \ \ \ \text{for} \ z \triangleq R_{kh\to(k-\ell)h}(\wtx_{kh}) \ \ \text{and} \ \ \gamma: \ D^{\gamma}_{(k-\ell)\to kh}(z) = \wtx_{kh}\,,
    \label{eq:mle}
\end{equation}
where $R$ and $D$ were defined in \eqref{eq:restore} and \eqref{eq:degrade} respectively. As $z$ is the result of restoring the current iterate $\wtx_{kh}$, we have
\begin{equation}
    z = \wtx_{kh} - \ell h \, f_{kh}(\wtx_{kh}) + \ell h \, g(kh)^2 \nabla \ln q_{kh}(\wtx_{kh}). \label{eq:zdef}
\end{equation}
The next iterate $\wtx_{(k-1)h}$ is given by degrading $z$ for time $(\ell - 1)h$, with the noise vector taken to be the \emph{simulated noise} $\gamma$. More precisely $\gamma$ is the noise vector that one could have used to degrade $z$ for time $\ell h$ to obtain $\wtx_{kh}$. As $\gamma$ is the solution to $D^\gamma_{(k-\ell) h\to kh}(z) = \wtx_{kh}$, an equivalent formulation is via
\begin{equation}
    \gamma = \frac{\wtx_{kh} - z - \ell h \, f_{(k-\ell)h}(z)}{g((k - \ell)h) \sqrt{\ell h}}. \label{eq:gamdef}
\end{equation}

Note that \eqref{eq:mle} is not well-defined when $k < \ell$; in this case, we take the update according to the Euler-Maruyama discretization:
\begin{equation}
    \wtx_{(k-1)h} = \wtx_{kh} - h(f_{kh}(\wtx_{kh}) - \frac{1}{2}g(kh)^2\nabla \ln q_{kh}(\wtx_{kh})) \ \ \ \text{if} \ k < \ell\,.
\end{equation}
It will be convenient to denote
\begin{equation}
    \wtxl_{kh} \triangleq \wtx_{T-kh}
\end{equation}
in the sequel.

% We note that the stochastic samplers
% As in the proof of Theorem~\ref{thm:main}, all of the approximations above can be made quantitative by adapting the analysis in Sections~\ref{sec:generic} and~\ref{sec:discrete_analysis}.

\subsection{Statement of results}
\label{sec:results} 

We make the following mild assumptions on the forward process $(x_t)$ and the data distribution:

\begin{assumption}\label{assume:smooth}
    For all $t\ge 0$, the following holds for parameters $L_{f;\mathsf{t}}, L_g, L_{f;\mathsf{x}}, L_{{\sf sc},t}, R, \gmax, \beta, M \ge 1, c> 0$:
    \begin{enumerate}
        \item $f_t(x)$ is $L_{f;\mathsf{t}}$-Lipschitz in $t$ and $L_{f;\mathsf{x}}$-Lipschitz in $x$. \label{item:Lfx}
        \item $g^2(t)$ is $L_g$-Lipschitz in $t$.
        \item $\norm{f_t(0)} \le R$.
        \item $g(t) \le \gmax$.
        \
        % \item The score functions are 2-dissipative, that is, for all $x\in\R^d$, $a\,\norm{x}^2 - b \le -\iprod{\nabla \ln q^\leftarrow_t(x), x}$. \label{item:dissipative}
        \item $\nabla \ln q^\leftarrow_t(x)$ is $\Lsc{t}$-Lipschitz in $x$ and satisfies
        \begin{equation}
            \norm{\nabla \ln\frac{q^\leftarrow_t}{q^\leftarrow_s}(x)} \le \beta |t - s|^c ( 1 + \norm{x} + \norm{\nabla q^\leftarrow_t(x)} )
        \end{equation} for all $s \ge 0$.  Denote $\sup_{t\ge 0} \Lsc{t}$ by $\Lsc{*}$. \label{item:score_lip}
        \item $\nabla f_t(x)$ and $\nabla^2 \ln q^\leftarrow_t$ are $\Lhigh$-Lipschitz in operator norm. \label{item:high}
    \end{enumerate}
\end{assumption}

\begin{remark}
    We note that the first four Parts of Assumption~\ref{assume:smooth}, as well as the first half of Part~\ref{item:high}, are quite mild and are satisfied by any reasonable choice of forward process. For instance, for the Ornstein-Uhlenbeck process $\d x_t = -x_t\,\d t + \sqrt{2} \d W_t$, we can take $\Lft = 0, \Lfx = 1, \Lg = 0, R = 0, \gmax = \sqrt{2}$, and $\nabla f_t(x) = -\Id$ for all $x$ is thus clearly Lipschitz in operator norm. Part~\ref{item:score_lip} ensures that the score functions $\nabla \ln q^\leftarrow_t$ do not change much when perturbed in space or time. The former is a standard assumption in the literature on discretization bounds for score-based generative modeling \cite{blomrorak2020generative,chen2022sampling,leelutan2022generative,leelutannew,chen2022improved}, and the latter holds for reasonable choices of forward process. For instance, for the Ornstein-Uhlenbeck process, we can take $c = 1/2$ and $\beta = \Theta(\Lsc{*}\sqrt{d})$ (see e.g. Lemma C.12 from \cite{leelutan2022generative}). 
    
    % Part~\ref{item:dissipative} is a bit more restrictive than existing analyses of score-based generative modeling which focus on stochastic samplers, though it is still quite mild and does not preclude significant non-log-concavity of $q$. Rather, it merely ensures some quadratic growth of the potentials of $q^\leftarrow_t$. It is only needed in one particular place in the proof, which we highlight in Remark~\ref{remark:dissipative_why}.

    The main distinction between Assumption~\ref{assume:smooth} and the assumptions made in previous analyses for score-based generative models is the second half of Part~\ref{item:high} where we assume \emph{higher-order smoothness} of $q^\leftarrow_t$. As we will see in Section~\ref{sec:generic}, this is essential to our analysis because third-order derivatives of $\ln q^\leftarrow_t$ naturally arise when one computes the time derivative of the Fisher information as described in Section~\ref{sec:overview}. As discussed in that section, the need to compute such time derivatives is unique to the ODE setting, justifying why such an assumption was not needed in prior analysis of stochastic samplers.
\end{remark}

\noindent Under these conditions, we show that our discretization procedure approximates the true reverse process to prescribed error $\epsilon$ provided $\ell$ and $(\ell h)^{-1}$ are larger than some quantities which are polynomially bounded in $1/\epsilon$ and all parameters from Assumption~\ref{assume:smooth}:

\begin{theorem}\label{thm:main}
    Let $\epsilon > 0$. Let $\wt{p}$ denote the law of the process $(\wtxl_{kh})$ at time $T$. 
    % Let $p$ denote the law of the process $(\xh_{kh})$ given by Euler-Maruyama discretization at time $0$~\eqref{eq:euler}. 
    Suppose Assumption~\ref{assume:smooth} holds and define
    \begin{equation}
        \Lambda \triangleq \exp\left(\int^T_0 (\Lfx^2 + \gmax^2 \Lsc{t})\,\d t\right) \qquad \text{and} \qquad \Lambda' \triangleq \exp\left(\int^T_0 (\Lfx^2 + \gmax^2 \Lsc{\lfloor t/h\rfloor h})\,\d t\right)\,. \label{eq:Lamdef}
    \end{equation}
    Then there exist quantities $\mathfrak{C}_1$ and $\mathfrak{C}_2$ which are polynomially bounded in $\Lft$, $\Lfx$, $\Lg$, $R$, $\gmax$, $\beta$, $\Lsc{*}$, $\Lhigh$, $\Lambda$, $\Lambda'$, $d$, $\mathbb{E}\norm{\xl_0}^2$, and $1/\epsilon$ such that $\KL{\wt{p}}{q} \le \epsilon$ provided $\ell \ge \mathfrak{C}_1$ and $\ell h \le \mathfrak{C}^{-1}_2$.
\end{theorem}

\begin{remark}
    We briefly remark on the quantities $\Lambda, \Lambda'$ appearing in the above theorem. We typically think of $\Lfx$ and $\gmax$ as of constant order, so $\Lambda$ and $\Lambda'$ scale polynomially with $\exp(\int^T_0 \Lsc{t}\, \d t)$ and $\exp(\int^T_0 \Lsc{\lfloor t/h\rfloor h}\, \d t)$. While this scales exponentially in $T$, the exponential convergence of reasonable forward processes like Ornstein-Uhlenbeck means we should think of $T$ as scaling logarithmically in $d/\epsilon$. And while naively one might suspect that $\Lambda, \Lambda'$ scale exponentially with $\Lsc{*}$, we show in Example~\ref{example:Lam} in Appendix~\ref{sec:generic} that these quantities actually scale polynomially in $d$ and other parameters like $\Lsc{*}$, e.g. when the data distribution is Gaussian. Altogether, this suggests that our non-asymptotic guarantees are of polynomial complexity in all relevant parameters from Assumption~\ref{assume:smooth}.
\end{remark}

\noindent In practice, the process $(\wtxl_{kh})$ would be initialized at the stationary measure $q^*$ of the forward process (after some suitable re-scaling), rather than at $q^\leftarrow_0$. As observed in \cite{leelutan2022generative,chen2022sampling,leelutannew}, the KL divergence between the final iterate of the process under the alternative initialization $\wtxl_0 \sim q^*$  and the final iterate under the initialization $\wtxl_0 \sim q^\leftarrow_0$ is at most the KL divergence between the initial iterates of these two processes. But by stationarity of $q^*$, the latter KL is equivalent to the KL between the stationary measure of the forward process and the the law of the forward process at time $T$. This KL is typically exponentially small in $T$, e.g. when the forward process is an Ornstein-Uhlenbeck process. By passing from KL to total variation via Pinsker's inequality and applying triangle inequality, we conclude that the total variation between $\wtxl_T$ under this alternative initialization and the data distribution $q$ is at most the sum of the error bound in Theorem~\ref{thm:main} plus the distance between $q_T$ and the stationary distribution. Formally, we obtain the following:

\begin{corollary}
    Let $\epsilon > 0$. Let $f_t(x) = -x$ and $g(t) = \sqrt{2}$, so that the forward process in \eqref{eq:forward} corresponds to the standard Ornstein-Uhlenbeck process. Define the process $(\overline{x}_{kh})$ to be the process given by the same updates as in \eqref{eq:mle} but with $\overline{x}_T$ sampled from $\calN(0,\Id)$ instead of $q^\leftarrow_0$. Let $p$ denote the law of $\overline{x}_0$.
    Suppose $\nabla^2 \ln q^\leftarrow_t$ is $\Lhigh$-Lipschitz in operator norm, and define
    \begin{equation}
        \Lambda \triangleq \exp\left(\int^T_0 \Lsc{t}\,\d t\right) \qquad \text{and} \qquad \Lambda' \triangleq \exp\left(\int^T_0 \Lsc{\lfloor t/h\rfloor h}\,\d t\right)\,.
    \end{equation}
    Then there exist quantities $\mathfrak{C}_1$ and $\mathfrak{C}_2$ which are polynomially bounded in $d$, $\Lsc{*}$, $\Lhigh$, $\Lambda$, $\Lambda'$, and $1/\epsilon$ such that
    \begin{equation}
        \TV(p,q) \le \epsilon + \sqrt{\KL{q}{\calN(0,\Id)}}\,\exp(-T)
    \end{equation}
    provided $\ell \ge \mathfrak{C}_1$ and $\ell h \le \mathfrak{C}^{-1}_2$.
\end{corollary}

\subsection{Proof overview}
\label{sec:overview}
% In what follows, we give a non-asymptotic analysis of the 

Our discretization analysis is an interpolation-style argument, similar to the kind used in the log-concave sampling literature \cite{vempala2019ulaisoperimetry,chewietal2021lmcpoincare,wibisono2022convergence} as well as some recent analyses of score-based generative modeling \cite{leelutan2022generative,leelutannew,chen2022improved}. Here we describe the setup for this argument and highlight the key technical differences that manifest when analyzing ODEs rather than SDEs.

We begin with a generic setting where we are given two stochastic processes ${(y_t)}_{t\in [0,T]}$ and $(y'_t)_{t\in[0,T]}$ as follows. The process $(y_t)$ is given by an arbitrary ODE
\begin{equation}
    \d y_t = \mu_t(y_t) \, \d t\,. \label{eq:cts_overview}
\end{equation}
We will ultimately take $\mu_t$ to be $-f_{T - t} + \frac{1}{2}g(T - t)^2 \nabla \ln q_{T-t}$ so that \eqref{eq:cts_overview} is the probability flow ODE associated to the forward process in \eqref{eq:forward}.
The process $(y'_t)$ is given by first taking a discrete-time approximation to $(y_t)$, e.g. via the update rules
\begin{equation*}
    y'_{(k+1)h} = y'_{kh} + h\cdot\mu'_{kh}(y_{kh})
\end{equation*}
for all integers $k = 0, 1, \ldots, T/h$. We will ultimately take $\mu'_{kh}$ to be $-f_{T-kh} + \frac{1}{2}g(T - kh)^2 \nabla \ln q_{T - kh}$ plus error terms coming from the approximations in \eqref{eq:firstapprox_overview} and from taking $\ell \to \infty$ (see Appendix~\ref{app:proof_prelims} for the explicit form for these error terms).

Then to get $y'_t$ for all real values $t\in [0,T]$, we consider a linear interpolation of these iterates: if $k = \lfloor t/h\rfloor$, then we define $y_t = y_{kh} + (t - kh)\mu'_{kh}(y_{kh})$. We write this as
\begin{equation*}
    \d y'_t = \mu'_{kh}(y'_{kh})\, \d t\,. \label{eq:discreteode_overview}
\end{equation*}

Provided these processes are both initialized at the same distribution, that is, $y_0, y'_0 \sim \pi$ for some probability measure $\pi$ over $\R^d$, then we would like to control the statistical distance between the marginal distributions on $y_t$ and on $y'_t$ as a function of $t$. Denoting these distributions by $\pi_t$ and $\pi'_t$ respectively, we prove the following generic bound which is the technical core of our work. First, we make the following assumptions about the two processes. When we specialize these processes to $(\xl_t)$ and $(\wtxl_t)$, these assumptions will follow from Assumption~\ref{assume:smooth}:

\begin{restatable}{assumption}{genericassume}\label{assume:generic}
    For all $0 \le t \le T$, there are parameters $L_t, L'_t, M \ge 1$ and $\zeta_t > 0$ such that:
    \begin{enumerate}
        \item $\nabla \ln \pi_t$ and $\mu_t$ are $L_t$-Lipschitz.\label{item:smoothcts_1} %and $\sup\norm{\nabla \mu_t - \Id}_{\sf op} \le L_t$. 
        \item $\nabla \mu_t$ is $M$-Lipschitz in operator norm. \label{item:smoothcts_2}
        \item $\mu'_t$ is $L'_t$-Lipschitz. \label{item:smoothdiscrete} % $\sup\norm{\nabla \mu'_t - \Id}_{\sf op} \le L'_t$. 
        \item $\E{\norm{\mu_t(y'_t) - \mu'_{kh}(y'_{kh})}^2} \le \zeta^2_t$. \label{item:error}
        \item $h \le 1/2L'_t$ for all $0 \le t \le T$. \label{item:smallstep}
    \end{enumerate}
\end{restatable}

\noindent We briefly interpret these assumptions in the context of our eventual application to bounding the error of our discretization procedure. There, Conditions~\ref{item:smoothcts_1} and~\ref{item:smoothcts_2} apply to the true continuous process. The former is an immediate consequence of our (standard) assumption on the second-order smoothness of the marginals of the true process. The latter is an immediate consequence of our assumption on the third-order smoothness, which is stronger than what is needed for analyses of the reverse \emph{SDE} but is likely necessary for our analysis of the reverse ODE.

Conditions~\ref{item:smoothdiscrete} and~\ref{item:error} are properties that we will eventually establish for our discretization procedure (see Section~\ref{sec:discrete_analysis}). Roughly, they stipulate that the drift term in the discretized probability flow ODE is Lipschitz and close on average to the drift of the true ODE.

Lastly, Condition~\ref{item:smallstep} simply corresponds to a constraint on the step size of our discretization procedure.

% \TODO{provide interpretation, i.e. 2 is standard, 3 is specific to ODE setting where we need higher-order smoothness, for 1 we only assume lipschitzness of drift rather than of the score of the algorithm, 4 is the main estimate that we will show in a later section, 5 is for invertibility of the flow map}

For convenience, we will also define the quantities
\begin{equation}
    L\triangleq \max_t L_t, \quad L'\triangleq \max_t L'_t, \quad \zeta^2 \triangleq \int^T_0 \zeta^2_t \, \d t, \quad 
    \Lambda \triangleq \exp\bigl(\int^T_0 L_t \,\d t\bigr), \quad \Lambda' \triangleq \exp\bigl(\int^T_0 L'_t \,\d t\bigr)\,.\label{eq:exponential_overview}
\end{equation}

% \TODO{include example calculation showing that $\Lambda$ scales roughly as $Ld$ for OU process starting at $\calN(0,\frac{1}{L}\Id)$.}

\noindent The main result of this section is a bound on the KL divergence between $\pi'_T$ and $\pi_T$:

\begin{restatable}{theorem}{klode}\label{thm:klode}
    \begin{equation}
        \KL{\pi'_T}{\pi_T} \lesssim \Lambda^{O(1)} L'^{1/2}\zeta^2 \\
        + (\Lambda^{O(1)} + \Lambda'^{O(1)}) (L'^{1/2}_0d^{1/2} + MdT^{1/2}) \, \zeta T^{1/2}\,.
    \end{equation}
\end{restatable}

The main ingredient in proving this is to bound the time derivative of $\KL{\pi'_t}{\pi_t}$ uniformly across $t\in [0,T]$, from which a bound on $\KL{\pi'_t}{\pi_t}$ follows by integrating.

One can explicitly compute this time derivative by appealing to the time derivatives of the densities of $\pi'_t, \pi_t$, given by the Fokker-Planck equations for the two processes:
\begin{align*}
    \partial_t \pi_t &= -\div(\pi_t\cdot \mu_t), \\
    \partial_t \pi'_t &= -\div(\pi'_t\cdot \wh{\mu}_{t,kh})\,,
\end{align*}
for $\wh{\mu}_{t,kh}(x) \triangleq \E{\mu'_{kh}(y'_{kh}) \mid y'_t = x}$. Here $\wh{\mu}_{t,kh}$ is the expectation over the drift at time $kh$ conditioned on the position at the \emph{future} time $t$. A calculation (see Lemma~\ref{lem:KLderiv}) then reveals that
\begin{equation}
    \partial_t \KL{\pi'_t}{\pi_t} = \int \pi'_t \iprod{\nabla \ln \frac{\pi'_t}{\pi_t}, \wh{\mu}_{t,kh} - \mu_t}\,. \label{eq:deriv_overview}
\end{equation}
Thus far, these are all standard steps. Here however, our analysis departs from usual applications of the interpolation method. Indeed, if the ODEs driving $y_t$ and $y'_t$ were SDEs equipped with an additional Brownian motion term, then \eqref{eq:deriv_overview} would come with an additional negative term given by a multiple of the \emph{Fisher information} between $\pi'_t$ and $\pi_t$. In equations, this means that in lieu of \eqref{eq:deriv_overview}, we would have
\begin{equation}
    \partial_t \KL{\pi'_t}{\pi_t} = \int \pi'_t \iprod{\nabla \ln \frac{\pi'_t}{\pi_t}, \wh{\mu}_{t,kh} - \mu_t} - C \int \pi'_t \norm{\nabla \ln \frac{\pi'_t}{\pi_t}}^2, \label{eq:deriv2_overview}
\end{equation}
for some $C > 0$ depending on the amount of Brownian motion.
The advantage of the Fisher information term in \eqref{eq:deriv2_overview} is that we can apply Young's inequality to conveniently upper bound the above by a multiple of
\begin{equation}
    \int \pi'_t \,\norm{\wh{\mu}_{t,kh} - \mu_t}^2, \label{eq:drift_diff}
\end{equation}
and avoid having to deal with $\nabla \ln \frac{\pi'_t}{\pi_t}$ altogether. Roughly speaking, the quantity~\eqref{eq:drift_diff} corresponds to the expected squared difference between the drift of the discrete process at time $kh$ versus the drift of the continuous process at time $t$. This is small provided the former process doesn't move around too much between times $kh$ and $t$, and provided the drifts $\mu'_{kh}$ and $\mu_t$ are sufficiently close on average. We verify in Section~\ref{sec:discrete_analysis} that both of these conditions are satisfied by the probability flow ODE.

The situation is trickier in the ODE setting. To handle \eqref{eq:deriv_overview}, we instead apply Cauchy-Schwarz to get
\begin{equation}
    \partial_t \KL{\pi'_t}{\pi_t} \le \left(\int \pi'_t \norm{\nabla \ln \frac{\pi'_t}{\pi_t}}^2\right)^{1/2} \cdot \left(\int \pi'_t \norm{\wh{\mu}_{t,kh} - \mu_t}\right)^{1/2}\,, \label{eq:cauchy_schwarz_overview}
\end{equation}
after which the main technical obstacle is to ensure the first term on the right-hand side, again corresponding to the Fisher information between $\pi'_t$ and $\pi_t$, does not explode with $t$. In Lemmas~\ref{lem:laplace} and \ref{lem:laplace2}, we show how to bound the time derivative of this quantity polynomially in various problem-specific parameters like dimension and smoothness of $\mu_t$. Altogether, this leads to the following bounds. We defer the technical details to the supplement and provide a brief proof sketch of how to control the time derivatives of these quantities:
\begin{lemma}[See Lemmas~\ref{lem:laplace} and \ref{lem:laplace2}]\label{lem:fisher_deriv}
    For all $0 \le t \le T$,
    \begin{align*}
        \E[\pi'_t]{\norm{\nabla \ln \pi'_t}^2} &\lesssim \Lambda'^{O(1)} (L'_0 d + M^2 d^2 t) \\
        \E[\pi'_t]{\norm{\nabla \ln \pi_t}^2} &\lesssim \Lambda^{O(1)} (L'_0 d + M^2 d^2 t + L'\zeta^2)
    \end{align*}
\end{lemma}

\begin{proof}[Proof sketch]
    When computing $\partial_t \int \pi'_t \norm{\nabla \ln \pi'_t}^2$, one term that shows up is $\partial_t\ln \pi'_t$. Using the Fokker-Planck equation for $\pi'_t$, we can derive an expression for $\partial_t \ln \pi'_t$ (see Proposition~\ref{prop:ln_fp}). This and a calculation with integration by parts reveals that \begin{align}
        \partial_t \int \pi'_t \norm{\nabla \ln \pi'_t}^2 &= -2\int \pi'_t \bigl(\iprod{\nabla \div\, \wh{\mu}_{t,kh},\nabla \ln \pi'_t} + (\nabla \ln \pi'_t)^\top(\nabla \wh{\mu}_{t,kh}) (\nabla \ln \pi'_t)\bigr) \\
        &\lesssim \sup_x\norm{\nabla \wh{\mu}_{t,kh}(x)}_{\sf op}\,\int \pi'_t \norm{\nabla \ln \pi'_t}^2 + \int \pi'_t \norm{\nabla \div\, \wh{\mu}_{t,kh}}^2,
    \end{align}
    where in the last step we used Young's inequality. Lipschitzness of $\mu_t$ allows us to bound $\sup\norm{\nabla \wh{\mu}_{t,kh}}_{\sf op}$, and higher-order smoothness of $\mu_t$ allows us to bound $\norm{\nabla \div \wh{\mu}_{t,kh}}$.
\end{proof}
This is the only part of the analysis where third-order derivatives appear and where Part~\ref{item:high} of Assumption~\ref{assume:smooth}, which corresponds to Part~\ref{item:smoothcts_2} of Assumption~\ref{assume:generic}, comes into play. One subtlety in the argument above is deducing smoothness of $\wh{\mu}_{t,kh}$, a complicated-looking conditional expectation, from smoothness of the true drift $\mu_t$. To connect the two, we exploit the fact that for step size $h$ sufficiently small, the discrete-time process is \emph{invertible} (Lemma~\ref{lem:smoothreverse1}) so that $\wh{\mu}_{t,kh}$ can be expressed as $\mu'_{kh}$ composed with a deterministic function.

Altogether, the bound on the Fisher information which is implied by Lemma~\ref{lem:fisher_deriv} allows us, as in the SDE case, to reduce controlling $\partial_t \KL{\pi'_t}{\pi_t}$ to controlling the difference in drifts as captured by Eq.~\eqref{eq:drift_diff}, which we then carry out in Appendix~\ref{sec:discrete_analysis}.

\section{Related Work}
\label{sec:related}
% \paragraph{Non-asymptotic guarantees for diffusion models.} 

There has been great recent progress on diffusion models including recently outperforming other deep generative models such as Generative Adversarial Networks (GANs)~\cite{dhariwal2021diffusion, song2020score, daras_dagan_2022score, ddpm++}. 
Applications range from protein generation~\cite{anand2022protein, trippe2022diffusion, schneuing2022structure, corso2022diffdock}, medical imaging~\cite{mri_paper, arvinte2022single}, 3-D data~\cite{poole2022dreamfusion} and many more, e.g. see \cite{survey_diffusion_1} for a comprehensive survey.

Non-asymptotic analysis of stochastic samplers, \cite{blomrorak2020generative, debetal2021scorebased, deb2022manifold, liu2022let, leelutan2022generative, pid2022manifolds, leelutannew, chen2022improved, chen2022sampling} has drawn upon tools from the rich literature on log-concave sampling (see \cite{chewisamplingbook} for a recent survey) to yield convergence guarantees for diffusion models. These works focus on the setting where the forward process is an Ornstein-Uhlenbeck process, and the reverse process is given by a \emph{stochastic} differential equation. Notably, the very recent works of \cite{chen2022sampling, leelutannew, chen2022improved} show under mild assumptions on the data distribution $q$ (e.g. smooth and bounded second moment) that a suitable discretization of the reverse SDE run for polynomially many steps generates samples that are close in statistical distance to the data distribution.

Prior to our work, no previous non-asymptotic bounds were known for the probability flow ODE associated to any forward process. Prior non-asymptotic analyses for diffusion models are insufficient because they either rely on Girsanov's theorem \cite{chen2022sampling} or a chain rule-based variant \cite{chen2022improved} of the interpolation argument of \cite{vempala2019ulaisoperimetry}, both of which yield vacuous bounds in the deterministic sampler setting as we now briefly explain.

Informally, Girsanov's theorem allows one to bound not just the distance between the distributions over the final iterates of the algorithm but even the distance between the distributions over the \emph{trajectories} of the two processes \--- note that the latter distance upper bounds the former by the data processing inequality. Stochasticity in every step of the reverse process ensures that even the latter distance is small. Without stochasticity however, this distance is infinite.

The chain rule-based argument of \cite{chen2022improved} establishes a similar bound to Girsanov's; in particular, when the algorithm and the idealized process are initialized to the same distribution, the bounds these two arguments give are identical.

Lastly, we remark that \cite{leelutan2022generative} used an interpolation argument without chain rule, but their analysis, similar to existing analyses of Langevin Monte Carlo in the log-sampling literature \cite{vempala2019ulaisoperimetry,chewietal2021lmcpoincare}, exploits the appearance of a certain Fisher information term in the expression for the time derivative of the KL divergence between the algorithm and the idealized process. As we explained in Section~\ref{sec:overview}, this Fisher information term does not appear for ODEs.

% \paragraph{DDIM type samplers}

\section{Conclusion}

In this work we gave an operational interpretation for the probability flow ODE as iterating a two-step process of restoration via gradient ascent and degradation towards the current iterate. This perspective also extends to reverse processes with a Brownian motion component. Our operational interpretation closely aligns with the samplers introduced in \cite{bansal2022cold,soft_diffusion} and generalizes the framework of denoising diffusion implicit models~\cite{ddim} to general, non-linear forward processes.

The main technical contribution of our work was a non-asymptotic analysis of the deterministic sampler arising from our framework. While previous works~\cite{chen2022sampling,leelutannew,chen2022improved} gave non-asymptotic analyses for diffusion models when the underlying reverse process is an SDE, to our knowledge our analysis is the first of its kind in the ODE setting. Our proof is based on an interpolation argument, but the key difference with prior applications of this method is that the deterministic nature of the sampler necessitates controlling the time derivative of the Fisher information between the algorithm and the true reverse process.

\paragraph{Limitations and future directions.} The most obvious area for improvement would be to sharpen the quantitative dependence on various parameters like dimension and Lipschitz-ness of the score functions. Intuitively, the absence of Brownian motion in the probability flow ODE should lead to better dimension dependence compared to using an SDE, but in this work we are only able to establish an iteration complexity for the deterministic sampler which is some polynomial in $d$. Additionally, for convenience in this work we ignore issues of score estimation error. While it should be possible to use change-of-measure-type arguments like in \cite{wibisono2022convergence} to obtain guarantees when the score estimation error has sub-Gaussian tails, new ideas are needed to handle merely an $L_2$ bound on the score estimation error like in \cite{chen2022sampling,leelutannew,chen2022improved}.

We also leave as an open question whether our assumption of higher-order smoothness is really necessary to obtain non-asymptotic guarantees for the probability flow ODE.

Apart from these technical improvements, we mention some empirical directions to explore. First, our discretization procedure introduces a number of new hyperparameters that one can try tuning to get improved performance in practice. Even for linear diffusions, it would be interesting to explore the effect of tuning $\ell$, which under DDIM is currently taken to be $(T-t) / h$. In addition, it seems interesting to explore how parameters of the restoration procedure like the learning rate and number of steps of gradient ascent, or the use of momentum or higher-order optimization methods can lead to better samplers. We expect that different restoration procedures can recover other discretization frameworks, e.g. second-order ones like Heun's method. Empirically, we expect that optimizing the learning rate and number of steps can lead to deterministic samplers with smaller computational overhead and higher sample quality. 

\paragraph{Acknowledgments.} SC would like to thank Sinho Chewi, Holden Lee, Yuanzhi Li, Jianfeng Lu, and Adil Salim for many enlightening discussions about deterministic score-based generative modeling and the interpolation method. The authors would also like to thank Sinho Chewi, Holden Lee, and Adil Salim for helpful feedback on an earlier version of this work.

This research has been supported by NSF Grants CCF 1763702,
AF 1901292, CNS 2148141, Tripods CCF 1934932, IFML CCF 2019844, the Texas Advanced Computing Center (TACC) and research gifts by Western Digital, WNCG IAP, UT Austin Machine Learning Lab (MLL), Cisco and the Archie Straiton Endowed Faculty Fellowship.
SC has been supported by NSF Award 2103300. GD has been supported by the Onassis Fellowship, the Bodossaki Fellowship and the Leventis Fellowship.

\bibliographystyle{alpha}
\bibliography{biblio}

%%%%%%%%%%%%%%%%%%%%%%%%%%%%%%%%%%%%%%%%%%%%%%%%%%%%%%%%%%%%%%%%%%%%%%%%%%%%%%%
%%%%%%%%%%%%%%%%%%%%%%%%%%%%%%%%%%%%%%%%%%%%%%%%%%%%%%%%%%%%%%%%%%%%%%%%%%%%%%%
% APPENDIX
%%%%%%%%%%%%%%%%%%%%%%%%%%%%%%%%%%%%%%%%%%%%%%%%%%%%%%%%%%%%%%%%%%%%%%%%%%%%%%%
%%%%%%%%%%%%%%%%%%%%%%%%%%%%%%%%%%%%%%%%%%%%%%%%%%%%%%%%%%%%%%%%%%%%%%%%%%%%%%%
\appendix

\section{Tuning the Learning Rate}
\label{app:learning_rate}

In this section we justify the choice of learning rate~\eqref{eq:lr} in our gradient ascent interpretation of the restoration operator by considering the special case where the data distribution $q$ is isotropic Gaussian and the forward process is an Ornstein-Uhlenbeck process. 

First, recall the definition of the loss function $\ell_{\xl_t}$ from \eqref{eq:bayes}. In general, one step of gradient ascent with learning rate $\eta$ starting from $\xl_t$ gives the iterate
\begin{equation}
    \xl_t + \eta \nabla \ell_{\xl_t}(\xl_t) = \xl_t + \eta\Bigl(-\frac{1}{g(T - t)^2}\bigl(\Id + (t - s) \nabla f_{T - s}(\xl_t)\bigr) f_{T - s}(\xl_t) + \nabla \ln q^\leftarrow_{s}(\xl_t)\Bigr)\,. \label{eq:nextiter}
\end{equation}
Now suppose $q \sim \calN(0,\sigma^2\, \Id)$ and furthermore
\begin{equation}
    f_t(x) = -\alpha x \qquad \text{and} \qquad g(t) = \beta\sqrt{2}\,.
\end{equation}
Then $q^\leftarrow_t$ is given by $\calN(0,(e^{-2\alpha t}\sigma^2 + \frac{\beta^2}{\alpha}(1 - e^{-2\alpha t}))\,\Id)$, and the conditional log-likelihood $\ln q^\leftarrow_s(x \mid x^\leftarrow_t)$ is quadratic in $x$ and is thus maximized at $x$ for which $\nabla \ell_{\xl_t}$ vanishes.

In this case, \eqref{eq:gradient} simplifies to
\begin{equation}
    \nabla\ell_{\xl_t} \approx \frac{1}{2\beta^2(t-s)}(1 + \alpha(s - t))(\xl_t - x(1 + \alpha(s - t))) - \frac{1}{\Var{q^\leftarrow_s}}\,x\,,
\end{equation}
where $\Var{q^\leftarrow_s}$ denotes the variance of $q^\leftarrow_s$.
Setting the right-hand side to zero and solving for $x$ shows that
\begin{equation}
    x = \frac{1 + \alpha(s - t)}{\Var{q_{T - s}} + (1+\alpha(s-t))^2} \, \xl_t = \Bigl(1 + \bigl(\alpha - \frac{2\beta^2}{\Var{q_{T - s}}}\bigr)\cdot(t - s) + O(|t-s|^2)\Bigr)\, \xl_t \label{eq:stat}
\end{equation}
is an (approximate) stationary point of $\nabla \ell_{\xl_t}$. 

The next iterate~\eqref{eq:nextiter} after one gradient step simplifies to
\begin{equation}
    \xl_t + \eta\nabla \ell_{\xl_t}(\xl_t) \approx \xl_t - \frac{\eta}{2\beta^2(t-s)}(1 + \alpha(s-t))\cdot \alpha(s - t)\cdot \xl_t - \frac{\eta}{e^{-2\alpha s}\sigma^2 + \frac{\beta^2}{\alpha}(1 - e^{-2s})}\, \xl_t.
\end{equation}
Finally, we observe that by taking 
\begin{equation}
    \eta \triangleq 2\beta^2(t - s)\,, \label{eq:eta_prelim}
\end{equation}
the above simplifies to
\begin{equation}
     \xl_t - (1 + \alpha(s - t))\cdot \alpha(s - t)\cdot \xl_t  - \frac{2\beta^2(t-s)}{e^{-2\alpha s} \sigma^2 + \frac{\beta^2}{\alpha}(1 - e^{-2s})}\, \xl_t = \Bigl(1 + \bigl(\alpha - \frac{2\beta^2}{\Var{q_{T - s}}}\bigr)\cdot(t - s) + O(|t-s|^2)\Bigr)\, \xl_t\,,
\end{equation}
which agrees up to second-order terms with \eqref{eq:stat}.
Therefore, when the data is Gaussian and the forward process is an Ornstein-Uhlenbeck process, as $t - s \to 0$ the right choice of $\eta$ to ensure to first order approximation that a single gradient step takes us from $\xl_t$ to the maximizer of the conditional log-likelihood $\ln q^\leftarrow_s(x\mid \xl_t)$ is given by~\eqref{eq:eta_prelim}, which corresponds to \eqref{eq:lr} in the main text as claimed.

% The above discussion motivates us to consider the following operation for general $f_t, g(t)$. In~\eqref{eq:nextiter}, we take
% \begin{equation}
%     \eta \triangleq g(t)^2\cdot (t - s)
% \end{equation}

\section{Proof Preliminaries}
\label{app:proof_prelims}

Let $h > 0$ and $\ell \in\mathbb{N}$ be discretization parameters. Define 
\begin{equation}
    \delta_\ell \triangleq 1 - \sqrt{1 - 1/\ell} = \frac{1}{2\ell} + O(1/\ell^2)
\end{equation}
and
\begin{equation}
    \xi_\ell = \delta_\ell - \frac{1}{2\ell} = O(1/\ell^2)
\end{equation}
% We will denote the $O(1/\ell^2)$ term by $\xi_\ell$ so that $\xi_\ell = 1 - \sqrt{1 - 1/\ell} - (2\ell)^{-1}$.
Recall the definition of the process $(\wtx_{kh})_{k\in\brc{0,\ldots,T/h}}$ in Eq.~\eqref{eq:mle}. Here we rewrite the update rule in \eqref{eq:mle} to make clear its similarity to the Euler-Maruyama discretization:
\begin{align}
    \wtx_{(k-1)h} &= z + (\ell-1) h \, f_{(k - \ell)h}(z) + g((k - \ell)h)\sqrt{(\ell - 1)h}\cdot \gamma \\
    % &= \wtx_{kh} + \ell h (\Id + kh \nabla f_{(k-\ell)h}(\wtx_{kh}))f_{kh}(\wtx_{kh}) - \ell h \, g(kh)^2 \nabla \ln q_{(k-\ell)h}(\wtx_{kh}) \\
    % &\quad\ \ \   + (\ell-1) h \, f_{(k - \ell)h}(z) + g((k - \ell)h)\sqrt{(\ell - 1)h}\cdot \frac{\wtx_{kh} - z - \ell h \, f_{(k-\ell)h}(z)}{g((k-\ell)h) \sqrt{\ell h}}
    &= z + (\ell-1) h \, f_{(k - \ell)h}(z) + g((k - \ell)h)\sqrt{(\ell - 1)h}\cdot \frac{\wtx_{kh} - z - \ell h \, f_{(k-\ell)h}(z)}{g((k-\ell)h) \sqrt{\ell h}} \\
    &= (1 - \delta_\ell)\wtx_{kh} + \delta_\ell z + (\ell\delta_\ell - 1)h f_{(k-\ell)h}(z) \\
    % \wtx_{(k-1)h} &= \wtx_{kh} + \bigl(z + (\ell - 1)h \, f_{(k-\ell)h}(z) + g((k - \ell)h)\sqrt{(\ell - 1)h}\cdot \gamma\bigr) \\
    % &\quad\quad - (1 - \delta_\ell) \, \bigl(z + \ell h \, f_{(k-\ell)h}(z) + g((k - \ell)h)\sqrt{\ell h}\cdot \gamma\bigr) \\
    % &= \wtx_{kh} + \delta_\ell z + (\ell\delta_\ell - 1) \, h f_{(k-\ell)h}(z) \\
    &= \wtx_{kh} + (\ell\xi_\ell - 1/2) \, h f_{(k-\ell)h}(z) \\
    &\quad\quad - \delta_\ell \bigl(\ell h \, f_{kh}(\wtx_{kh}) - \ell h \, g(kh)^2\nabla \ln q_{kh}(\wtx_{kh})\bigr).
\intertext{Note that $\delta_\ell \eta_k = h g(kh)^2 \,(1/2 + \ell\xi_\ell)$, so we can further rewrite this as}
    &= \wtx_{kh} + (\ell\xi_\ell - 1/2) \, h f_{(k-\ell)h}(z) \\
    &\quad\quad - h(1/2+\ell\xi_\ell) \bigl(f_{kh}(\wtx_{kh}) - g(kh)^2 \nabla \ln q_{kh}(\wtx_{kh})\bigr) \\
    &= \wtx_{kh} - h\,\brc{f_{kh}(\wtx_{kh})- \frac{1}{2} g(kh)^2 \nabla \ln q_{kh}(\wtx_{kh})} + \vecv^{(1)}_{kh}(\wtx_{kh}) + \cdots + \vecv^{(3)}_{kh}(\wtx_{kh}),
\end{align}
where the excess terms are given by
\begin{align}
    % \vecv^{(1)}_{kh}(\wtx_{kh}) &\triangleq \delta_\ell \wtx_{kh} \\
    \vecv^{(1)}_{kh}(\wtx_{kh}) &\triangleq \ell\xi_\ell h f_{(k-\ell)h}(z)\cdot \bone{k \ge \ell} \\
    \vecv^{(2)}_{kh}(\wtx_{kh}) &\triangleq \frac{h}{2}(f_{(k-\ell)h}(z) - f_{kh}(\wtx_{kh}))\cdot \bone{k \ge \ell} \\
    \vecv^{(3)}_{kh}(\wtx_{kh}) &\triangleq h\ell\xi_\ell\bigl(-f_{kh}(\wtx_{kh}) + g(kh)^2 \nabla \ln q_{kh}(\wtx_{kh})\bigr)\bone{k \ge \ell}\cdot \,.
    % \vecv^{(4)}_{kh}(\wtx_{kh}) &\triangleq -\ell h^2(1/2 + \ell \xi_\ell)\nabla f_{(k-\ell)h}(\wtx_{kh}) f_{kh}(\wtx_{kh}) \\
    % \vecv^{(5)}_{kh}(\wtx_{kh}) &\triangleq \frac{1}{2} hg(kh)^2(\nabla \ln q_{(k-\ell)h}(\wtx_{kh}) - \nabla \ln q_{kh}(\wtx_{kh})).
\end{align}
% Alternatively, when $k < \ell$, we define these excess terms to be 0.
% 1 is lipschitzness of f, bound on ||z - \wtx_{kh}||^2, and bound on ||f(\wtx_{kh})||^2
% 2 is lipschitzness of f and bound on ||z - \wtx_{kh}||^2
% 3 is bound on ||f(\wtx_{kh})||^2 and bound on score
% 4 is lipschitzness of f and bound on ||f(\wtx_{kh})||^2
% 5 is score perturbation lemma

Note that as $\ell \to \infty$ and $h\ell\to 0$, the excess terms tend to zero and the process $(\wtx_{kh})$ converges to the one given by the Euler-Maruyama discretization.

In the subsequent sections, we make this quantitative via an interpolation argument.
Let $(\wtx_t)_{0\le t \le T}$ denote the linear interpolation of the discrete process $(\wtx_{kh})_{h = 0,\ldots,T/h}$, and let $(\wtxl_t)$ denote the time-reversed process $\wtxl_t \triangleq \wtx_{T-t}$. Concretely, for any $kh \le t < (k+1)h$,
% \begin{align}
%     \d \wtxl_t &= \frac{1}{h}\bigl(D^\gamma_{(k-\ell)h\to(k-1)h}(z) - \sqrt{1 - 1/\ell}\, D^\gamma_{(k-\ell)h\to kh}(z)\bigr)\,\d t\,, \ \ \ z\triangleq R^{\eta_k}_{kh\to(k-\ell)h}(\wtxl_{kh}) \\ % f_{kh}(\wtxl_{kh}) + \frac{1}{2}g(kh)^2 \nabla \ln q_{kh}(\wtxl_{kh})
%     &= -\brc{f_{T-kh}(\wtxl_{kh}) - \frac{1}{2}g(T - kh)^2\nabla \ln q^\leftarrow_{kh}(\wtxl_{kh}) + \frac{1}{h}(\vecv^{(1)}_{T-kh}(\wtxl_{kh}) + \cdots + \vecv^{(5)}_{T-kh}(\wtxl_{kh}))}\,\d t.\quad \label{eq:ourdiscrete}
% \end{align}
\begin{align}
    \d \wtxl_t &= -\Bigl\{f_{T-kh}(\wtxl_{kh}) - \frac{1}{2}g(T - kh)^2\nabla \ln q^\leftarrow_{kh}(\wtxl_{kh}) \\ &\qquad\qquad - \frac{1}{h}\,
    % \bone{(k+\ell)h \le T}\,
    (\vecv^{(1)}_{T-kh}(\wtxl_{kh}) + \cdots + \vecv^{(3)}_{T-kh}(\wtxl_{kh}))\Bigr\}\,\d t.\label{eq:ourdiscrete}
\end{align}
We note that even in the absence of the excess terms above, in which case the above process would just be the Euler-Maruyama discretization of the probability flow ODE, no existing works gave a non-asymptotic analysis showing that this discretization converges polynomially to the continuous-time probability flow ODE. Our analysis in the sequel allows us to both control the excess terms and establish such a non-asymptotic analysis.

\section{Interpolation Argument}
\label{sec:generic}

In this section we give general bounds for how the KL divergence between two distributions, one driven by a discretized ODE and the other by a continuous-time one, changes over time. Throughout this section, we work with two stochastic processes ${(y_t)}_{t\in [0,T]}$ and $(y'_t)_{t\in[0,T]}$  over $\R^d$ given by the ODEs
\begin{align}
    \d y_t &= \mu_t(y_t) \, \d t \label{eq:cts} \\
    \d y'_t &= \mu'_{kh}(y'_{kh}) \, \d t, \ \ \  k = \lfloor t/h\rfloor,  \label{eq:dis}
\end{align}
where $y_0, y'_0 \sim \pi$ for some probability measure $\pi$ over $\R^d$. The process $(y'_t)$ is equivalent to a linear interpolation of a discrete-time process where one goes from the $k$-th iterate $y'_{kh}$ to the $(k+1)$-st iterate $y'_{(k+1)h}$ via the update 
\begin{equation}
    y'_{(k+1)h} = y'_{kh} + h\, \mu'_{kh}(y'_{kh})\,.
\end{equation}

We let $\pi_t,\pi'_t$ denote the law of $y_t, y'_t$ respectively. When we eventually apply the estimates obtained in this section, we will take $(y'_t)$ to be given by our discretization of the probability flow ODE, and we will take $(y_t)$ to be the true probability flow ODE in continuous time.

The bounds in this section hold under the conditions of Assumption~\ref{assume:generic}, restated here for convenience:

\genericassume*

% \TODO{provide interpretation, i.e. 2 is standard, 3 is specific to ODE setting where we need higher-order smoothness, for 1 we only assume lipschitzness of drift rather than of the score of the algorithm, 4 is the main estimate that we will show in a later section, 5 is for invertibility of the flow map}

\noindent For convenience, we also recall the quantities defined in \eqref{eq:exponential_overview}:
\begin{equation}
    L\triangleq \max_t L_t, \quad L'\triangleq \max_t L'_t, \quad \Lambda \triangleq \exp\bigl(\int^T_0 L_t \,\d t\bigr), \quad \Lambda' \triangleq \exp\bigl(\int^T_0 L'_t \,\d t\bigr), \quad \zeta^2 \triangleq \int^T_0 \zeta^2_t \, \d t
\end{equation}
and restate the main claimed bound on the KL divergence between $\pi'_T$ and $\pi_T$:
\klode*

\begin{example}\label{example:Lam}
    Here we work out a simple example showing that when $(y_t)$ corresponds to the probability flow ODE that reverses the Ornstein-Uhlenbeck process starting from a Gaussian distribution, $\Lambda'$ scales polynomially, rather than exponentially, in $d$ and $L'$.
    
    Define $\pi^{\rightarrow}_t$ for $0 \le t \le T$ as the marginal distribution of running the Ornstein-Uhlenbeck process for time $t$ starting from $\calN(0,\frac{1}{L}\Id)$ for some large $L$, and consider the associated reverse ODE
    \begin{equation}
        \d y_t = (y_t + \nabla \ln \pi_t(y_t))\, \d t,
    \end{equation}
    where $\pi_t \triangleq \pi^{\rightarrow}_{T-t}$ denotes the marginal laws of $(y_t)_{t\in[0,T]}$. Concretely, $\pi_t$ is given by $\calN(0,\frac{1}{L_t}\Id)$ for $L_t = (e^{-2(T-t)}/L + 1 - e^{-2(T-t)})^{-1}$. Note that
    \begin{equation}
        \Lambda' = \exp\bigl(\int^T_0 L_t \, \d t\bigr) = \exp\bigl(\frac{1}{2}\ln(1 + (e^{2T} - 1)L)\bigr).
    \end{equation}
    Because $\KL{\calN(0,\frac{1}{L}\Id)}{\calN(0,\Id)} = \frac{d}{2}(\ln L - 1 + \frac{1}{\ell}) \lesssim d\ln L$, we must run the forward process for time $T\approx \frac{1}{2}\ln (d\ln L)$ for $\pi^\rightarrow_T$ to be close to $\calN(0,\Id)$. In this case, $\Lambda' \lesssim \sqrt{d L\ln L}$.
\end{example}

\noindent We begin by working out the Fokker-Planck equations for $(\pi'_t)$ and $(\pi_t)$.

\begin{proposition}\label{prop:fp}
    The laws $(\pi'_t)$ and $(\pi_t)$ satisfy
    \begin{align}
        \partial_t \pi_t &= -\div(\pi_t\cdot \mu_t)  \\
        \partial_t \pi'_t &= -\div(\pi'_t\cdot \wh{\mu}_{t,kh}),
    \end{align}
    where
    \begin{equation}
        \wh{\mu}_{t,kh}(x) \triangleq \E{\mu'_{kh}(y'_{kh}) \mid y'_t = x}.
    \end{equation}
\end{proposition}

\noindent When $k$ is clear from context, we will denote $\wh{\mu}_{t,kh}$ by $\wh{\mu}_t$ to ease notation.

\begin{proof}
    The Fokker-Planck equation for $(\pi_t)$ is given by 
    \begin{equation}
        \partial_t \pi_t = - \div(\pi_t\cdot \mu_t).
    \end{equation}
    For the interpolated process $(\pi'_t)$, the Fokker-Planck for $(\pi'_t)_{kh\le t < (k+1)h}$ conditioned on time $kh$, which we will denote by $(\pi'_{t|kh})_{kh\le t < (k+1)h}$, is given by
    \begin{equation}
        \partial_t \pi'_{t|kh}(x) = - \div_{x}(\pi'_{t|kh}(x) \cdot \mu'_{kh}(y'_{kh})). \label{eq:conditional}
    \end{equation}
    If $\Pi'_{kh}$ denotes the probability measure over $\sigma(y'_t \mid 0\le t \le kh)$, then if we integrate both sides of \eqref{eq:conditional} with respect to $\Pi'_{kh}$, we get
    \begin{align}
        \partial_t \pi'_t(x) &= - \int \div_{x}(\pi'_t(x \mid \xi) \cdot \mu'_{kh}(y'_{kh})) \, \Pi'_{kh}(\d \xi) \\
        &= -\div_{x} \int \pi'_t(x \mid \xi) \cdot \mu'_{kh}(y'_{kh}) \, \Pi'_{kh}(\d \xi) \\
        &= -\div_{x} \bigl(\pi'_t(x) \int \mu'_{kh}(y'_{kh}) \, \Pi'_{kh|t}(\d\xi\mid y'_t = x)\bigr) \\
        &= -\div_{x}(\pi'_t(x)\cdot \E{\mu'_{kh}(y'_{kh}) \mid y'_t = x}) \\
        &= -\div_x(\pi'_t(x)\cdot \wh{\mu}_{t,kh}(x)). \qedhere
    \end{align}
\end{proof}

\noindent It turns out that because we are assuming the step size $h$ is sufficiently small in Condition~\ref{item:smallstep} of Assumption~\ref{assume:generic}, the conditional expectation $\wh{\mu}_{t,kh}$ has a simple form. For any $k$, the ODE $\d y'_t = \mu'_{kh}(y'_{kh})\, \d t$ defines a map $F_{kh\to t}: \R^d\to\R^d$ for any $kh \le t \le (k+1)h$ via
\begin{equation}
    F_{kh\to t}(z) = z + (t - kh)\mu'_{kh}(z)
\end{equation}
so that starting at $z$ at time $kh$ and running the ODE to time $t$, we end up at $F_{kh\to t}(z)$. When $h$ is sufficiently small, $F_{kh\to t}$ is invertible:
\begin{lemma}\label{lem:smoothreverse1}
    Let $h \le 1/2L'$. Then for any $z,z'\in\R^d$,
    \begin{equation}
        \frac{1}{2}\norm{z - z'} \le \norm{F_{kh\to t}(z) - F_{kh\to t}(z')} \le \frac{3}{2}\norm{z - z'}. \label{eq:bilip}
    \end{equation}
    In particular, $F_{kh\to t}$ has a unique, $2$-Lipschitz inverse $F^{-1}_{kh\to t}: \R^d\to\R^d$, so 
    \begin{equation}
        \wh{\mu}_{t,kh}(x) = \mu'_{kh}(F^{-1}_{kh\to t}(x)). \label{eq:muhat_inverse}
    \end{equation}
    Furthermore, $\wh{\mu}_{t,kh}$ is $O(L'_t)$-Lipschitz.
\end{lemma}

\noindent Henceforth, when $k,h,t$ are clear from context, we will refer to the inverse $F^{-1}_{kh\to t}$ simply as $F^{-1}$.

\begin{proof}
     For the first bound, note that
     \begin{equation}
        \norm{F_{kh\to t}(z) - F_{kh\to t}(z')} \ge \norm{z - z'} - (t - kh)\norm{\mu'_{kh}(z) - \mu'_{kh}(z')} \ge (1 - h\cdot L'_{kh})\, \norm{z - z'},
     \end{equation}
     so the lower bound in \eqref{eq:bilip} follows by the fact that $h \le 1/2L'$. The upper bound follows analogously. 
     
     For the second part of the lemma, recall that bi-Lipschitz functions on $\R^d$ are bijective, so $F_{kh\to t}$ has a unique inverse $F^{-1}_{kh\to t}$. To see why the latter function is 2-Lipschitz, for any $z_0, z'_0$ we can take $z = F^{-1}_{kh\to t}(z_0)$ and $z' = F^{-1}_{kh\to t}(z'_0)$ in the lower bound of \eqref{eq:bilip} to conclude that $\frac{1}{2}\norm{F^{-1}_{kh\to t}(z_0) - F^{-1}_{kh\to t}(z'_0)} \le \norm{z_0 - z'_0}$ as desired. Eq.~\eqref{eq:muhat_inverse} then follows from the fact that the distribution of $y'_{kh}$ conditioned on $y'_t = x$ is the point mass at $F^{-1}_{kh\to t}(x)$.
     
     The only part that remains to be verified is Lipschitzness of $\wh{\mu}_{t,kh}$. This follows from the fact that $\wh{\mu}_{t,kh}$ is the composition of a $L'_t$-Lipschitz function with a $2$-Lipschitz function.
    %  For any $x, x'$, we have
    % \begin{align}
    %     \MoveEqLeft \norm{(\wh{\mu}_{t,kh}(x) - x) - (\wh{\mu}_{t,kh}(x') - x')} \\
    %     &\le \norm{(\mu'_{kh}(F^{-1}(x)) - F^{-1}(x)) - (\mu'_{kh}(F^{-1}(x')) - F^{-1}(x')} + \norm{(F^{-1}(x) - x) - (F^{-1}(y) - y)} \\
    %     &\le L'\norm{F^{-1}(x) - F^{-1}(y)} + \norm{(F^{-1}(x) - F(F^{-1}(x))) - (F^{-1}(x') - F(F^{-1}(x'))} \\
    %     &\le L'(1 - hL')^{-1}\norm{x - y} + (t - kh)\norm{\mu'_{kh}(F^{-1}(x) - \mu'_{kh}(F^{-1}(y)} \\
    %     &\le (1 - hL')^{-1}(L' + h(L' + 1))\norm{x - y},
    % \end{align}
\end{proof}

\noindent We will also use the following simple consequence of the third-order smoothness of $\mu_t$ (Condition~\ref{item:smoothcts_2} of Assumption~\ref{assume:generic}):

\begin{lemma}\label{lem:smoothreverse2}
    For all $x,x'\in\R^d$, then
    \begin{equation}
        \sup\norm{\nabla \div\, \mu_t} \le Md \qquad \text{and} \qquad \sup \norm{\nabla \div \,\wh{\mu}_{t,kh}} \le 2 Md
    \end{equation}
\end{lemma}

\begin{proof}
    The first bound is immediate from
    \begin{equation}
        |\div\,\mu_t(x) - \div\,\mu_t(x')| = |\Tr\nabla\,\mu_t(x)- \Tr\nabla\, \mu_t(x')| \le \norm{\nabla \mu_t(x) - \nabla \mu_t(x')}_{\sf tr} \le M d \norm{x - x'}.
    \end{equation}
    For the second bound, note that 
    \begin{equation}
        |\div\,\wh{\mu}_t(x) - \div\,\wh{\mu}_t(x')| \le \norm{\nabla \mu'_{kh}(F^{-1}(x)) - \nabla \mu'_{kh}(F^{-1}(x'))}_{\sf op} \le d M \norm{F^{-1}(x) - F^{-1}(x')} \le 2M d
    \end{equation}
    as claimed.
\end{proof}

\noindent We are now ready to compute the time derivative of the KL divergence between $\pi'_t$ and $\pi_t$.

\begin{lemma}\label{lem:KLderiv}
    \begin{equation}
        \partial_t \KL{\pi'_t}{\pi_t} \le \zeta_t \, \bigl(\int \pi'_t \norm{\nabla \ln \pi'_t - \nabla \ln \pi_t}^2 \bigr)^{1/2}
    \end{equation}
\end{lemma}

\begin{proof}
    We can compute
    \begin{align}
        \partial_t \KL{\pi'_t}{\pi_t}  &= \int (\partial_t \pi'_t) \ln \frac{\pi'_t}{\pi_t} + \int \pi'_t \, \partial_t \ln \frac{\pi'_t}{\pi_t} = \int (\partial_t \pi'_t) \ln \frac{\pi'_t}{\pi_t} + \int \pi'_t \, \frac{\partial_t (\pi'_t/\pi_t)}{\pi'_t / \pi_t} \\
        &= \int (\partial_t \pi'_t) \ln \frac{\pi'_t}{\pi_t} + \int \pi_t \cdot \frac{\pi_t \partial_t \pi'_t - \pi'_t \partial_t \pi_t}{{\pi_t}^2} \\
        &= \int (\partial_t \pi'_t) \ln \frac{\pi'_t}{\pi_t} - \int \frac{\pi'_t}{\pi_t}\, \partial_t \pi_t \\
        &= -\int \div(\pi'_t\cdot \wh{\mu}_{t,kh}) \, \ln\frac{\pi'_t}{\pi_t} + \int \frac{\pi'_t}{\pi_t} \, \div(\pi_t \cdot \mu_t) \\
        &= \int \pi'_t \, \iprod{\wh{\mu}_{t,kh}, \nabla \ln \frac{\pi'_t}{\pi_t}} - \int \pi_t \, \iprod{\nabla \frac{\pi'_t}{\pi_t}, \mu_t} \\
        &= \int \pi'_t \, \iprod{\nabla \ln \frac{\pi'_t}{\pi_t}, \wh{\mu}_{t,kh} - \mu_t}.
    \end{align}
    The lemma then follows by Cauchy-Schwarz, as 
    \begin{align}
        \int \pi'_t \norm{\wh{\mu}_{t,kh} - \mu_t}^2 &= \E[\pi'_t]{\norm{\mu'_{kh}(F^{-1}(y'_t)) - \mu_t(y'_t)}^2} = \E[\pi'_t]{\norm{\mu'_{kh}(y'_{kh}) - \mu_t(y'_t)}^2} \le \zeta^2_t. \qedhere
    \end{align}
\end{proof}

\noindent We need to control the Fisher information $\int \pi'_t\norm{\nabla \ln \pi'_t - \nabla \ln \pi_t}^2$ in Lemma~\ref{lem:KLderiv}. To do this, we will bound the time derivatives of $\int \pi'_t\norm{\nabla \ln \pi'_t}^2$ and $\int \pi'_t\norm{\nabla \ln \pi_t}^2$ in Lemmas~\ref{lem:laplace} and~\ref{lem:laplace2} below and apply triangle inequality.

\begin{lemma}\label{lem:laplace}
    \begin{equation}
        \partial_t \int \pi'_t \norm{\nabla \ln \pi'_t}^2 \lesssim L'_t\, \int \pi'_t \norm{\nabla \ln \pi'_t}^2 + M^2 d^2
    \end{equation}
    In particular, by Gr\"{o}nwall's inequality, for any $0 \le t \le T$ we have
    \begin{equation}
        \int \pi'_t\norm{\nabla \ln \pi'_t}^2 \lesssim \Lambda'^{O(1)} (L'_0 d + M^2 d^2 t) \label{eq:gronwall1}
    \end{equation}
\end{lemma}

\begin{proof}
    % For convenience, denote $\wh{\mu}_{t,kh}$ by $\wh{\mu}_t$. Then
    We have
    \begin{align}
        \partial_t \int \pi'_t \norm{\nabla \ln \pi'_t}^2 &= -\int \div(\pi'_t\cdot \wh{\mu}_t) \norm{\nabla \ln \pi'_t}^2 + \int \pi'_t \partial_t \norm{\nabla \ln \pi'_t}^2 \\
        &= 2\int \pi'_t \bigl(\iprod{\wh{\mu}_t, (\nabla^2 \ln \pi'_t)\nabla \ln \pi'_t} + \iprod{\partial_t \nabla \ln \pi'_t, \nabla \ln \pi'_t} \bigr)\\
        &= 2\int \pi'_t \bigl(\iprod{\wh{\mu}_t, (\nabla^2 \ln \pi'_t)\nabla \ln \pi'_t} + \iprod{\nabla(-\div\,\wh{\mu}_t - \iprod{\nabla \ln \pi'_t, \wh{\mu}_t}), \nabla \ln \pi'_t}\bigr),\label{eq:lap_deriv}
    \end{align}
    where in the last step we used the first part of Proposition~\ref{prop:ln_fp} below.
    Note that we can write the latter term in the parentheses in \eqref{eq:lap_deriv} as
    \begin{equation}
        \iprod{-\nabla \div\,\wh{\mu}_t - (\nabla^2 \ln \pi'_t)\wh{\mu}_t - (\nabla \wh{\mu}_t)\nabla \ln \pi'_t, \nabla \ln \pi'_t}.
    \end{equation}
    Of these three terms, the second one exactly cancels with the first term in \eqref{eq:lap_deriv}. Putting everything together, we get
    \begin{align}
        \partial_t \int \pi'_t \norm{\nabla \ln \pi'_t}^2 &= -2\int \pi'_t \bigl(\iprod{\nabla \div\, \wh{\mu}_t,\nabla \ln \pi'_t} + (\nabla \ln \pi'_t)^\top(\nabla \wh{\mu}_t) (\nabla \ln \pi'_t)\bigr) \\
        &\lesssim \sup\norm{\nabla \wh{\mu}_t}_{\sf op}\,\int \pi'_t \norm{\nabla \ln \pi'_t}^2 + \int \pi'_t \norm{\nabla \div\, \wh{\mu}_t}^2,
    \end{align}
    where in the last step we used Young's inequality. The first part of the lemma follows by Lemmas~\ref{lem:smoothreverse1} and~\ref{lem:smoothreverse2}. For the second part, Gr\"onwall's inequality tells us that
    \begin{equation}
        \E[\pi'_t]{\norm{\nabla \ln \pi'_t}^2} \le \Lambda'^{O(1)} (\int \pi\norm{\nabla \ln \pi}^2 + M^2 d^2 t).
    \end{equation} 
    We conclude by noting that \begin{equation}
        \int \pi \norm{\nabla \ln \pi}^2 = -\int \pi\Delta \ln \pi \le L'_0 d \label{eq:generator}
    \end{equation}
    by integration by parts and Condition~\ref{item:smoothcts_1} of Assumption~\ref{assume:generic}.
\end{proof}

\noindent We remark that Lemma~\ref{lem:laplace} is tight as $h \to 0$ when the marginals $\brc{\pi'_{T-t}}_{t\in [0,T]}$ are given by running the Ornstein-Uhlenbeck process starting with a spherical Gaussian distribution. 

In the above proof, we needed the following calculation:
\begin{proposition}\label{prop:ln_fp}
    \begin{align}
        \partial_t \ln \pi'_t &= -\div\,\wh{\mu}_{t,kh} - \iprod{\nabla \ln \pi'_t, \wh{\mu}_{t,kh}} \\
        \partial_t \ln \pi_t &= -\div\,\mu_t - \iprod{\nabla \ln \pi_t, \mu_t}.
    \end{align}
\end{proposition}

\noindent Next, we carry out a calculation analogous to Lemma~\ref{lem:laplace} to bound the time derivative of $\E[\pi'_t]{\norm{\nabla \ln \pi_t}^2}$:

\begin{lemma}\label{lem:laplace2}
    \begin{equation}
        \partial_t \int \pi'_t \norm{\nabla \ln \pi_t}^2 \lesssim L_t \, \int \pi'_t \norm{\nabla \ln \pi_t}^2 + M^2 d^2 + L_t\zeta^2_t.
    \end{equation}
    In particular, by Gr\"{o}nwall's inequality, for any $0 \le t \le T$ we have
    \begin{equation}
        \E[\pi'_t]{\norm{\nabla \ln \pi_t}^2} \lesssim \Lambda^{O(1)} (L'_0 d + M^2 d^2 t + L'\zeta^2)\label{eq:gronwall2}
    \end{equation}
\end{lemma}

\begin{proof}
    We have
    \begin{align}
        \partial_t \int \pi'_t \norm{\nabla \ln \pi_t}^2 &= -\int \div(\pi'_t\cdot \wh{\mu}_t) \norm{\nabla \ln \pi_t}^2 + \int \pi'_t \partial_t \norm{\nabla \ln \pi_t}^2 \\
        &= 2\int \pi'_t \bigl(\iprod{\wh{\mu}_t, (\nabla^2 \ln \pi_t)\nabla \ln \pi_t} + \iprod{\partial_t \nabla \ln \pi_t, \nabla \ln \pi_t} \bigr)\\
        &= 2\int \pi'_t \bigl(\iprod{\wh{\mu}_t, (\nabla^2 \ln \pi_t)\nabla \ln \pi_t} + \iprod{\nabla(-\div\, \mu_t - \iprod{\nabla \ln \pi_t, \mu_t}), \nabla \ln \pi_t}\bigr),\label{eq:lap_derivprime}
    \end{align}
    where in the last step we used the second part of Proposition~\ref{prop:ln_fp}.
    Note that we can write the latter term in the parentheses in \eqref{eq:lap_derivprime} as
    \begin{equation}
        \iprod{-\nabla \div\, \mu_t - (\nabla^2 \ln \pi_t)\mu_t - (\nabla \mu_t)\nabla \ln \pi_t, \nabla \ln \pi_t}.
    \end{equation}
    Of these three terms, the second one nearly cancels with the first term in \eqref{eq:lap_derivprime}. Putting everything together, we get the inequality
    \begin{align}
        \partial_t \int \pi'_t \norm{\nabla \ln \pi_t}^2 &= -2\int \pi'_t \bigl(\iprod{\nabla \div\, \mu_t,\nabla \ln \pi_t} + (\nabla \ln \pi_t)^\top(\nabla \mu_t) (\nabla \ln \pi_t) \\
        &\qquad \qquad\qquad + (\mu_t - \wh{\mu}_t)^\top (\nabla^2\ln \pi_t) \nabla \ln \pi_t \bigr) \\
        &\lesssim \sup\norm{\nabla \mu_t}_{\sf op}\,\int \pi'_t \norm{\nabla \ln \pi_t}^2 + \int \pi'_t \norm{\nabla \div\,  \mu_t}^2 \\
        &\qquad \qquad \qquad + 2 \sup\norm{\nabla^2 \ln \pi_t}_{\sf op}\, \Bigl(\int \pi'_t \norm{\nabla \ln \pi_t}^2\Bigr)^{1/2} \Bigl(\int \pi'_t \norm{\mu_t - \wh{\mu}_t}^2\Bigr)^{1/2} \\
        &\lesssim L_t\int \pi'_t\norm{\nabla \ln \pi_t}^2 + \int \pi'_t \norm{\nabla \div\, \mu_t}^2 + L_t\int \pi'_t \norm{\mu_t - \wh{\mu}_t}^2
    \end{align}
    where in the penultimate and final steps we used Young's inequality, and in the final step we used Condition~\ref{item:smoothcts_1} of Assumption~\ref{assume:generic}. The first part of the lemma follows by Lemmas~\ref{lem:smoothreverse1} and Condition~\ref{item:error} of Assumption~\ref{assume:generic}. The second part of the lemma follows by Gr\"{o}nwall's inequality and \eqref{eq:generator}.
\end{proof}

\noindent We can now combine Lemmas~\ref{lem:KLderiv}, \ref{lem:laplace}, and \ref{lem:laplace2} to prove Theorem~\ref{thm:klode}:

\begin{proof}[Proof of Theorem~\ref{thm:klode}]
By triangle inequality and Eqs.~\eqref{eq:gronwall1} and \eqref{eq:gronwall2},
    \begin{equation}
        \bigl(\int \pi'_t\norm{\nabla \ln \pi'_t - \nabla \ln \pi_t}^2\bigr)^{1/2} \lesssim (\Lambda^{O(1)} + \Lambda'^{O(1)})(L'^{1/2}_0d^{1/2} + Mdt^{1/2}) + \Lambda^{O(1)} L'^{1/2}\zeta_t,
    \end{equation}
    so integrating the bound in Lemma~\ref{lem:KLderiv} over $t\in[0,T]$, we get
    \begin{equation}
        \KL{\pi'_T}{\pi_T} \lesssim (\Lambda^{O(1)} + \Lambda'^{O(1)})(L'^{1/2}_0d^{1/2} + MdT^{1/2})\,\int^T_0 \zeta_t \, \d t + \Lambda^{O(1)} L'^{1/2}\zeta^2\,.
    \end{equation}
    We conclude by bounding $\int^T_0 \zeta_t \, \d t \le \zeta T^{1/2}$ by Cauchy-Schwarz.
\end{proof}

Finally, we record a norm bound which will be useful in the sequel:

\begin{lemma}\label{lem:norm_deriv}
    For any $0 \le t \le T$ and any $c > 0$,
    \begin{equation}
        \partial_t\, \mathbb{E}\norm{y'_t}^2 \le \mathbb{E}\norm{\mu'_{kh}}^2 + \mathbb{E}\norm{y'_t}^2\,. %  = 4\,\E{\iprod{\mu'_{kh}(y'_{kh}),y'_t}}\,.
    \end{equation}
\end{lemma}

\begin{proof}
    Recall that $y'_t = y'_{kh} + (t - kh)\, \mu'_{kh}(y'_{kh})$, so
    \begin{equation}
        \mathbb{E}\norm{y'_t}^2 = \mathbb{E} \norm{y'_{kh}}^2 + (t - kh)^2\,\mathbb{E}\norm{\mu'_{kh}(y'_{kh})}^2 + 2(t-kh)\,\mathbb{E}\iprod{y'_{kh}, \mu'_{kh}(y'_{kh})}\,.
    \end{equation}
    Differentiating with respect to $t$, we get
    \begin{equation}
        \partial_t \mathbb{E}\norm{y'_t}^2 = 2(t - kh)\,\mathbb{E}\norm{\mu'_{kh}(y'_{kh})}^2 + 2\,\mathbb{E}\iprod{y'_{kh},\mu'_{kh}(y'_{kh})} = 2\,\mathbb{E}\iprod{y'_t, \mu'_{kh}(y'_{kh})}\,,
    \end{equation}
    so the lemma follows by Young's inequality.
    % We compute the time derivative 
    % \begin{align}
    %     \partial_t\, \mathbb{E}\norm{y'_t}^2 &= 
    %     \int \partial_t \pi'_t(y'_t) \cdot \norm{y'_t}^2 + 2\int \pi'_t(y'_t) \cdot \iprod{y'_t, \partial_t y'_t} \\
    %     &= -\int \div(\pi'_t(y'_t)\cdot \wh{\mu}_{t,kh}(y'_t))\cdot \norm{y'_t}^2 + 2\int \pi'_t(y'_t) \cdot \iprod{y'_t, \mu'_{kh}(y'_{kh})} \\
    %     &= \int \pi'_t(y'_t)\cdot \iprod{\wh{\mu}_{t,kh}(y'_t), \nabla \norm{y'_t}^2} + 2\int \pi'_t(y'_t)\cdot \iprod{y'_t, \mu'_{kh}(y'_{kh})} \\
    %     &= 4\,\E{\iprod{\mu'_{kh}(y'_{kh}),y'_t}} \\
    %     &\le 2\,\mathbb{E}\norm{\mu'_{kh}}^2 + 2\,\mathbb{E}\norm{y'_t}^2\,, %\int \pi'_t(y'_t)\cdot\iprod{\wh{\mu}_{t,kh}(y'_t), y'_t}\,. % \\
    %     % &\le 2\int \pi'_t(y'_t)\, \norm{y'_t}^2 + 2\int \pi'_t(y'_t)\,\norm{\wh{\mu}_{t,kh}}^2 \\
    %     %&= \frac{2}{c^2}\,\mathbb{E}\norm{y'_t}^2 + 2c^2\,\mathbb{E}\norm{\mu'_{kh}(y'_{kh})}^2\,,
    % \end{align}
    % where in the second step we used Proposition~\ref{prop:fp}, in the third step we used integration by parts, and in the fourth step we used Lemma~\ref{lem:smoothreverse1} to express $\mu'_{kh}(y'_{kh})$ as $\wh{\mu}_{t,kh}(y'_{kh})$, and in the fifth step we used Young's inequality.
    % The lemma follows by Gr\"{o}nwall.
\end{proof}

\section{Bounding the Difference in Drifts}
\label{sec:discrete_analysis}

We wish to apply Theorem~\ref{thm:klode} with $(y_t)$ and $(y'_t)$ given by $(x^\leftarrow_t)$ and $(\wt{x}^\leftarrow_t)$ defined in Eqs.~\eqref{eq:basicode} and~\eqref{eq:ourdiscrete}. For these processes, the drifts $(\mu_{kh})$ and $(\mu'_t)$ in Eqs.~\eqref{eq:cts} and \eqref{eq:dis} are given by
\begin{align}
    \mu_t(x) &\triangleq -f_{T-t}(x) + \frac{1}{2}g(T-t)^2 \nabla \ln q^\leftarrow_{t}(x) \label{eq:mut} \\
    \mu'_{kh}(x) &\triangleq -f_{T-kh}(x) + \frac{1}{2}g(T - kh)^2\nabla \ln q^\leftarrow_{kh}(x) - \frac{1}{h}(\vecv^{(1)}_{T-kh}(x) + \cdots + \vecv^{(3)}_{T-kh}(x))\,, \label{eq:muprimekh}
\end{align}
and both processes are initialized at the distribution $\pi = q_T$. In general, the marginal laws $(\pi_t)$ of the former process are given by $(q^\leftarrow_{t})$. We will denote the marginal laws $(\pi'_t)$ of the latter process by $(p_t)$.

\subsection{Smoothness of drift}

We now verify the first three parts of Assumption~\ref{assume:generic}. 
\begin{lemma}\label{lem:part1}
    Part~\ref{item:smoothcts_1} of Assumption~\ref{assume:generic} holds with 
    \begin{equation}
        L_t \triangleq \Theta(\Lfx + \gmax^2\Lsc{t})\,.\label{eq:Lt}
    \end{equation}
\end{lemma}

\begin{proof}
    By Part~\ref{item:score_lip} of Assumption~\ref{assume:smooth}, $\nabla \ln q^\leftarrow_t$ is $\Lsc{t}$-Lipschitz. As $\mu_t$ is the sum of an $\Lfx$-Lipschitz function and a $\frac{1}{2}\gmax^2\Lsc{t}$-Lipschitz function, the claim follows.
\end{proof}

\begin{lemma}\label{lem:part2}
    Part~\ref{item:smoothcts_2} of Assumption~\ref{assume:generic} holds with
    \begin{equation}
        M\triangleq (1 + \gmax^2/2)\Lhigh = \Theta(\gmax^2 \Lhigh)\,. \label{eq:M}
    \end{equation}
\end{lemma}

\begin{proof}
    By Part~\ref{item:high} of Assumption~\ref{assume:smooth}, $\nabla \mu_t$ is the sum of a $\Lhigh$-Lipschitz function and a $\gmax^2 \Lhigh / 2$-Lipschitz function.
\end{proof}

\begin{lemma}\label{lem:restore_lipschitz}
    The restoration operator $R_{kh\to(k-\ell)h}$ is $O(1)$-Lipschitz for all integers $\ell \le k \le T/h$.
\end{lemma}

\begin{proof}
    For any $x,x'$, we have
    \begin{align}
        \norm{R_{kh\to(k-\ell)h}(x) - R_{kh\to(k-\ell)h}(x')} &\le \norm{x - x'} + \ell h\,\norm{f_{kh}(x) - f_{kh}(x')} \\
        &\quad\quad+ \ell h\, g(kh)^2\,\norm{\nabla \ln q_{kh}(x) - \nabla \ln q_{kh}(x')} \\
        &\le (1 + \ell h \Lfx + \ell h \gmax^2 \Lsc{kh})\,\norm{x - x'} \lesssim \norm{x - x'}. \qedhere
    \end{align}
\end{proof}

\begin{lemma}\label{lem:excess_lipschitz}
    $\frac{1}{h}(\vecv_1 + \cdots + \vecv_3)$ is $O(\Lfx + \gmax^2\Lsc{kh})$-Lipschitz.
\end{lemma}

\begin{proof}
    By Lemma~\ref{lem:restore_lipschitz}, $f_{(k-\ell)h}(z) = f_{(k-\ell)h}(R_{kh\to (k-\ell h)}(\wtxl_{kh})$ is a composition of an $\Lfx$-Lipschitz function with an $O(1)$-Lipschitz function in $\wtxl_{kh}$, so $\frac{1}{h}\vecv_1$ is $O(\ell\xi_\ell \Lfx)$-Lipschitz. Similarly, $\frac{1}{h}\vecv_2$ is the difference between an $O(\Lfx)$-Lipschitz function and an $\Lfx/2$-Lipschitz function in $\wtxl_{kh}$, so it is $O(\Lfx)$-Lipschitz. Finally, $\frac{1}{h}\vecv_3$ is the sum of an $\ell\xi_\ell \Lfx \ll \Lfx$-Lipschitz function and a $\gmax^2 \Lsc{kh}$-Lipschitz function, so it is $(\Lfx + \gmax^2\Lsc{kh}$-Lipschitz.
\end{proof}

\begin{lemma}\label{lem:part3}
    $\mu'_{kh}$ as defined in \eqref{eq:muprimekh} is $O(\Lfx + \gmax^2\Lsc{kh})$-Lipschitz. In particular, Part~\ref{item:smoothdiscrete} of Assumption~\ref{assume:generic} holds with 
    \begin{equation}
        L'_t \triangleq \Theta(\Lfx + \gmax^2\Lsc{kh}) \label{eq:Lprimet}
    \end{equation} 
    for all $kh \le t < (k+1)h$.
\end{lemma}

\begin{proof}
    Note that $f_{T-kh}(\cdot) - \frac{1}{2}g(T-kh)^2\nabla \ln q^\leftarrow_{kh}(\cdot)$ is $O(\Lfx + \gmax^2\Lsc{kh})$-Lipschitz, so the claim follows by Lemma~\ref{lem:excess_lipschitz}.
\end{proof}

\subsection{Distance between drifts}

The bulk of our discretization analysis is devoted to verifying Part~\ref{item:error} of Assumption~\ref{assume:generic}. For convenience, we will denote $\vecv^{(1)}_{T-kh}(z),\ldots,\vecv^{(3)}_{T-kh}(z)$ by $\vecv_1,\ldots,\vecv_3$. Henceforth, assume that
\begin{equation}
    h \ll \min((R\Lfx \ell)^{-1}, (\gmax^2\Lsc{*})^{-1}) \label{eq:hassume} % \qquad \beta^2(\ell h)^{2c} \lesssim 1 \qquad ???\label{eq:hassume}
\end{equation}
For any $kh \le t \le (k+1)h$, we have
\begin{align}
    \mathbb{E}\norm{\mu_t(\wtxl_t) - \mu'_{kh}(\wtxl_{kh})}^2 &\lesssim \mathbb{E}\norm{f_{T-t}(\wtxl_t) - f_{T-kh}(\wtxl_{kh})}^2 \\
    &\quad \quad + \mathbb{E}\norm{g(T-t)^2\nabla\ln q^\leftarrow_t(\wtxl_t) - g(T-kh)^2\nabla\ln q^\leftarrow_{kh}(\wtxl_{kh})}^2\\
    &\quad \quad+ \frac{1}{h^2}(\mathbb{E}\norm{\vecv_1}^2 + \cdots \mathbb{E}\norm{\vecv_3}^2)\,. \label{eq:excess}
\end{align}

We first bound the excess terms $\vecv_1,\ldots,\vecv_3$. We focus on the case $T - kh \ge \ell h$, as otherwise $\vecv_1 = \vecv_2 = \vecv_3 = 0$ by definition.
\begin{lemma}\label{lem:excess}
    \begin{equation}
        \frac{1}{h^2}\E{\norm{\vecv_1}^2 + \cdots + \norm{\vecv_3}^2} \lesssim \epsilon_1\, \max_{k'\in\{0,1,\ldots,T/h\}}\mathbb{E}\norm{\nabla \ln q^\leftarrow_{k'h}(\wtxl_{k'h})}^2 + \epsilon_2
    \end{equation}
    for 
    \begin{align}
        % these are the bounds you get from dissipativity:
        % \epsilon_1 &\triangleq \ell^{-2}\gmax^4 + \ell^2 h^2 \gmax^4 \Lfx^2 + \frac{1}{a^2}(\ell^{-2}\Lfx^2 + \ell^2 h^2 \Lfx^4) \label{eq:epsdef1}\\ 
        % \epsilon_2 &\triangleq  \ell^2 h^2 \Lft^2 + \ell^{-2}R^2 + \ell^2 h^2 R^2\Lfx^2 + \frac{b}{a}(\ell^{-2}\Lfx^2 + \ell^2 h^2 \Lfx^4)\,.\label{eq:epsdef2}
        \epsilon_1 &\triangleq \exp(O(\Lfx^2 T)) (\ell^{-2} + \ell^2 h^2 \Lfx^2)\,\gmax^4 \label{eq:epsdef1}\\ 
        \epsilon_2 &\triangleq \exp(O(\Lfx^2 T)) (\ell^{-2} + \ell^2 h^2 \Lfx^2) (\mathbb{E}\norm{\wtxl_0}^2 + R^2 + \ell^2 h^2 \Lft^2)\,.\label{eq:epsdef2}
    \end{align}
\end{lemma}

\begin{proof}
Recall that 
\begin{equation}
    \vecv_1 = \ell\xi_\ell h f_{T - (k-\ell)h}(z)\,,
\end{equation}
so we have
\begin{align}
    \mathbb{E}\norm{\vecv_1}^2 &= \ell^2 \xi^2_\ell h^2 \, \mathbb{E}\norm{f_{T - (k - \ell)h}(z)}^2 \\
    &\lesssim \ell^{-2} h^2 (\mathbb{E}\norm{f_{T - (k-\ell)h}(\wtxl_{kh})}^2 + L^2_{f;\mathsf{x}} \, \mathbb{E}\norm{z - \wtxl_{kh}}^2) \\
    &\lesssim \ell^{-2} h^2 (L^2_{f;\mathsf{x}}\, \mathbb{E}\norm{\wtxl_{kh}}^2 + R^2 + L^2_{f;\mathsf{x}}\, \mathbb{E}\norm{z - \wtxl_{kh}}^2)\,.\label{eq:vec1}
\end{align}
Recall that
\begin{equation}
    \vecv_2 = \frac{h}{2}(f_{T-(k-\ell)h}(z) - f_{T-kh}(\wtxl_{kh}))\,,
\end{equation}
so we have
\begin{align}
    \mathbb{E}\norm{\vecv_2}^2 &= \frac{h^2}{4}\,\mathbb{E}\norm{f_{T-(k-\ell)h}(z) - f_{T-kh}(\wtxl_{kh})}^2 \\
    &\lesssim h^2(\ell^2 h^2 \Lft^2 + \Lfx^2\, \norm{z - \wtxl_{kh}}^2)\,.\label{eq:vec2}
    % &{\color{red} \text{(before, incorrect)} \lesssim L^2_{f;\mathsf{x}}\, \mathbb{E}\norm{z - \wtxl_{kh}}^2\,. 
\end{align}
Recall that
\begin{equation}
    \vecv_3 = h\ell\xi_\ell\bigl(-f_{T-kh}(\wtxl_{kh}) + g(T-kh)^2\nabla \ln q^\leftarrow_{(k+\ell)h}(\wtxl_{kh})\bigr)\,,
\end{equation}
so we have
\begin{align}
    \mathbb{E}\norm{\upsilon_3}^2 &= h^2\ell^2\xi_\ell^2 \, \mathbb{E} \norm{-f_{T-kh}(\wtxl_{kh}) + g(T-kh)^2 \nabla \ln q^\leftarrow_{kh}(\wtxl_{kh})}^2 \\
    &\lesssim \ell^{-2} h^2 (R^2 + L^2_{f;\mathsf{x}}\, \mathbb{E}\norm{\wtxl_{kh}}^2 + \gmax^4 \, \mathbb{E}\norm{\nabla \ln q^\leftarrow_{kh}(\wtxl_{kh})}^2)\label{eq:vec3}
    % \\
    % &\lesssim \ell^{-2} h^2 \Bigl(R^2 + \Lfx^2 \, \mathbb{E}\norm{\wtxl_{kh}}^2 + \gmax^4\mathbb{E}\norm{\nabla \ln q^\leftarrow_{kh}(\wtxl_{kh})}^2 \\
    % &\quad\quad + \gmax^4 \, \beta^2 (\ell h)^{2c}(1 + \mathbb{E}\norm{\wtxl_{kh}}^2 + \mathbb{E}\norm{\nabla q^\leftarrow_{kh}(\wtxl_{kh})}^2)
    % \Bigr) \\
    % &\lesssim \ell^{-2}h^2\Bigl(R^2 + \gmax^4\beta^2(\ell h)^{2c} + (\Lfx^2 + \gmax^4\beta^2(\ell h)^{2c})\,\mathbb{E}\norm{\wtxl_{kh}}^2 + \\
    % &\quad\quad +\gmax^4(1 + \beta^2(\ell h)^{2c})\,\mathbb{E}\norm{\nabla q^\leftarrow_{kh}(\wtxl_{kh})}^2\Bigr)
\end{align}
% Recall that
% \begin{equation}
%     \vecv_4 = -\ell h^2(1/2 + \ell\xi_\ell)\nabla f_{T - (k-\ell)h}(\wtxl_{kh})f_{T - kh}(\wtxl_{kh})\,,
% \end{equation}
% so we have
% \begin{align}
%     \mathbb{E} \norm{\upsilon_4}^2 &= \ell^2 h^4(1/2 +\ell\xi_\ell)^2 \, \mathbb{E}\norm{\nabla f_{T-(k - \ell)h}(\wtxl_{kh}) f_{T-kh}(\wtxl_{kh})}^2 \\
%     &\lesssim \ell^2 h^4 L^2_{f;\mathsf{x}} (L^2_{f;\mathsf{x}}\, \mathbb{E}\norm{\wtxl_{kh}}^2 + R^2)\,. \label{eq:vec4}
% \end{align}
% Recall that
% \begin{equation}
%     \vecv_5 = \frac{1}{2}hg(T-kh)^2(\nabla\ln q^\leftarrow_{(k - \ell)h}(\wtxl_{kh}) - \nabla \ln q^{\leftarrow}_{kh}(\wtxl_{kh}))\,,
% \end{equation}
% so we have
% \begin{align}
%     \mathbb{E}\norm{\upsilon_5}^2 &= \frac{1}{4}h^2g(T-kh)^4\,\mathbb{E}\norm{\nabla \ln q^\leftarrow_{(k-\ell)h}(\wtxl_{kh}) - \nabla \ln q^\leftarrow_{kh}(x_{T-kh})}^2 \\
%     &\lesssim h^2 \gmax^4 \beta^2 (\ell h)^{2c} \, (1 + \mathbb{E}\norm{\wtxl_{kh}}^2 + \mathbb{E}\norm{\nabla \ln q^\leftarrow_{kh}(\wtxl_{kh})}^2)\,.\label{eq:vec5}
% \end{align}
Combining Eqs.~\eqref{eq:vec1},~\eqref{eq:vec2}, and~\eqref{eq:vec3} we get
\begin{align}
    \frac{1}{h^2}\E{\norm{\vecv_1}^2 + \cdots + \norm{\vecv_3}^2} &\lesssim (\ell^2 h^2 \Lft^2 + \ell^{-2} R^2)  %+ \gmax^4\beta^2(\ell h)^{2c}) \\
    % &\quad\quad + (\gmax^4\beta^2(\ell h)^{2c} + \ell^{-2}\gmax^4)\, \mathbb{E}\norm{\nabla \ln q^\leftarrow_{kh}(\wtxl_{kh})}^2 \\
     + \ell^{-2}\gmax^4\,
    % \,(1 + \beta^2(\ell h)^{2c})\, 
    \mathbb{E}\norm{\nabla \ln q^\leftarrow_{kh}(\wtxl_{kh})}^2 \\
    % &\quad\quad  + (\ell^{-2} \Lfx^2 + \ell^2 h^2 \Lfx^4 + \gmax^4\beta^2(\ell h)^{2c})\, \mathbb{E}\norm{\wtxl_{kh}}^2 \\
    &\quad\quad + \ell^{-2}\Lfx^2 
    % + \gmax^4\beta^2(\ell h)^{2c})
    \, \mathbb{E}\norm{\wtxl_{kh}}^2 
    % &\quad\quad + \ell^2 h^2 \Lft^2 + \ell^{-2} R^2 + \ell^2 h^2 \Lfx^2 R^2 + \gmax^4\beta^2(\ell h)^{2c} \label{eq:ugly}
     + \Lfx^2\, \mathbb{E}\norm{z - \wtxl_{kh}}^2 \,.\label{eq:ugly}
\end{align}

Recall from \eqref{eq:zdef} that
\begin{equation}
    z = \wtx_{kh} - \ell h\, (f_{kh}(\wtx_{kh}) - g(T-kh)^2 \nabla \ln q^\leftarrow_{kh}(\wtx_{kh}))\,,
\end{equation}
so
\begin{align}
    \norm{z - \wtxl_{kh}}^2 &\lesssim \ell^2 h^2(1 + \ell^2 h^2 L^2_{f;\mathsf{x}}) \, \norm{f_{T-kh}(\wtxl_{kh})}^2 + \ell^2 h^2 \gmax^4 \, \norm{\nabla \ln q^\leftarrow_{kh}(\wtxl_{kh})}^2 \\
    &\lesssim \ell^2 h^2\, (L^2_{f;\mathsf{x}}\, \norm{\wtxl_{kh}}^2 + R^2) + \ell^2 h^2 \gmax^4 \, \norm{\nabla \ln q^\leftarrow_{kh}(\wtxl_{kh})}^2, \label{eq:zminusx}
    % \intertext{As $\norm{\nabla \ln q^\leftarrow_{(k-\ell)h}(\wtxl_{kh})}^2 \lesssim \norm{\nabla \ln q^\leftarrow_{(k-\ell)h}(\wtxl_{(k-\ell)h})}^2 + \Lsc{(k-\ell)h}^2\, \norm{\wtxl_{kh} - \wtxl_{(k-\ell)h}}^2$, this is at most}
    % \intertext{As $\norm{\nabla \ln q^\leftarrow_{(k-\ell)h}(\wtxl_{kh})}^2 \lesssim \norm{\nabla \ln q^\leftarrow_{kh}(\wtxl_{kh})}^2 + \beta^2 (\ell h)^{2c} (1 + \mathbb{E}\norm{\wtxl_{kh}}^2 + \mathbb{E}\norm{\nabla\ln q^\leftarrow_{kh}(\wtxl_{kh})}^2)$ and because of \eqref{eq:hassume}, this is at most}
    % &\lesssim \ell^2 h^2 \, \norm{f_{T-kh}(\wtxl_{kh})}^2 + \ell^2 h^2 \gmax^4 \, \norm{\nabla \ln q^\leftarrow_{kh}(\wtxl_{kh})}^2 \\
    % &\quad\quad + \ell^2 h^2 \gmax^4 \beta^2 (\ell h)^{2c}(1 + \mathbb{E}\norm{\wtxl_{kh}}^2 + \mathbb{E}\norm{\nabla\ln q^\leftarrow_{kh}(\wtxl_{kh})}^2) \\
    % % \ell^2 h^2 \gmax^4 \Lsc{(k-\ell)h}^2 \, \norm{\wtxl_{kh} - \wtxl_{(k-\ell)h}}^2 \\
    % &\lesssim \ell^2 h^2\, (L^2_{f;\mathsf{x}}\, \norm{\wtxl_{kh}}^2 + R^2) + \ell^2 h^2 \gmax^4 \, \mathbb{E}\norm{\nabla \ln q^\leftarrow_{kh}(\wtxl_{kh})}^2 \\
    % &\quad\quad + \ell^2 h^2 \gmax^4 \beta^2 (\ell h)^{2c}(1 + \mathbb{E}\norm{\wtxl_{kh}}^2) \\
    % &\lesssim \ell^2 h^2(R^2 + \gmax^4 \beta^2(\ell h)^{2c}) + \ell^2 h^2 \gmax^4 \,\mathbb{E}\norm{\nabla \ln q^\leftarrow_{kh}(\wtxl_{kh})}^2 \\
    % &\quad\quad + \ell^2 h^2(\Lfx^2 + \gmax^4\beta^2(\ell h)^{2c})\,\mathbb{E}\norm{\wtxl_{kh}}^2
    % % + \ell^2 h^2 \gmax^4 \Lsc{(k-\ell)h}^2 \, \norm{\wtxl_{kh} - \wtxl_{(k-\ell)h}}^2\,.
\end{align}
where in the second step we used \eqref{eq:hassume}.
Substituting this into~\eqref{eq:ugly} and using Lemma~\ref{lem:norm_deriv} below to bound $\mathbb{E}\norm{\wtxl_{kh}}^2$, we obtain the desired bound.
\end{proof}
% \TODO{fix later: the $\mathbb{E}\norm{\wtxl_{kh} - \wtxl_t}^2$ terms shouldn't be there, so we can probably ultimately take $\epsilon_1 = \epsilon'_1, \epsilon_2 = \epsilon'_2$. Similarly, the $h^2 L^2_g$ shouldn't be there, only appears when you look at difference between $\mu$'s}

% \begin{align}
%     \MoveEqLeft \mathbb{E}\norm{\mu_t(\wtxl_t) - \mu'_{kh}(\wtxl_{kh})}^2 \\
%     &\lesssim (\Lfx^2 + \gmax^4 \Lsc{t}^2)\, \mathbb{E}\norm{\wtxl_{kh} - \wtxl_t}^2 \\
%     &\quad\quad + (\ell^{-2} \Lfx^2 + \ell^2 h^2 \gmax^4 \Lsc{(k-\ell)h}^2 \Lfx^2) \, \mathbb{E}\norm{\wtxl_{(k-\ell)h} - \wtxl_{kh}}^2 \\
%     &\quad\quad + (h^2 \Lg^2 + \gmax^4\beta^2 h^{2c} + \gmax^4\beta^2(\ell h)^{2c}) \, \mathbb{E}\norm{\nabla \ln q^\leftarrow_{kh}(\wtxl_{kh})}^2 \\
%     &\quad\quad + \ell^2 h^2 \gmax^4 \Lfx^2 \, \mathbb{E}\norm{\nabla \ln q^\leftarrow_{(k-\ell)h}(\wtxl_{(k-\ell)h})}^2 \\
%     &\quad\quad  + (\gmax^4\beta^2 h^{2c} + \ell^{-2} \Lfx^2 + \ell^2 h^2 \Lfx^4 + \ell^4 h^4 \Lfx^6 + \gmax^4\beta^2(\ell h)^{2c})\, \mathbb{E}\norm{\wtxl_{kh}}^2 \\
%     &\quad\quad + h^2 \Lft^2 + \gmax^4\beta^2 h^{2c} + \ell^{-2} R^2 +  \ell^2 h^2 R^2 \Lfx^2 + \ell^4 h^4 \Lfx^4 R^2 + \gmax^4\beta^2(\ell h)^{2c}
% \end{align}

\begin{lemma}\label{lem:bound_drift_diff}
    For any integer $0 \le k \le T/h$ and any $kh \le t < (k + 1)h$,
    \begin{equation}
        \mathbb{E}\norm{\mu_t(\wtxl_t) - \mu'_{kh}(\wtxl_{kh})}^2 \lesssim \epsilon'_1\,\max_{k'\in\{0,1,\ldots,T/h\}}\mathbb{E}\norm{\nabla\ln q^\leftarrow_{k'h}(\wtxl_{k'h})}^2 + \epsilon'_2
    \end{equation}
    for
    \begin{align}
        % \epsilon'_1 &\triangleq h^2\Lg^2 + h^2 \gmax^8 \Lsc{t}^2 + \ell^{-2}\gmax^4 + \ell^2 h^2 \gmax^4 \Lfx^2 \\
        % &\quad\quad+ \frac{1}{a^2}(\ell^{-2}\Lfx^2 + \gmax^4\beta^2 h^{2c} + \ell^2 h^2 \Lfx^4 + h^2\Lfx^2\Lsc{t}^2\gmax^4) \label{eq:epsdef1_prime}\\ 
        % \epsilon'_2 &\triangleq \ell^2 h^2 \Lft^2 + \ell^{-2}R^2 + \ell^2 h^2 R^2\Lfx^2 + h^2\Lfx^2 R^2(\Lfx^2 + \gmax^4 \Lsc{t}^2)\\
        % &\quad\quad + \frac{b}{a}(\ell^{-2}\Lfx^2 + \gmax^4\beta^2h^{2c} + \ell^2 h^2 \Lfx^4 + h^2\Lfx^2\gmax^4\Lsc{t}^2)\,.\label{eq:epsdef2_prime}
        \epsilon'_1 &\triangleq \epsilon_1 + h^2\Lg^2 + \gmax^4(h^2\Lfx^2 + h^2\gmax^4\Lsc{*}^2 + \gmax^4 \beta^2 h^{2c})\cdot \exp(O(\Lfx^2 T)) \label{eq:epsdef1_prime}\\ 
        \epsilon'_2 &\triangleq \epsilon_2 + \gmax^4 \beta^2 h^{2c} + (\mathbb{E}\norm{\wtxl_0}^2 + R^2 + \ell^2 h^2 \Lft^2) \\
        &\quad\quad \times (h^2 \Lfx^2 + h^2\gmax^4\Lsc{*} + \gmax^4 \beta^2 h^{2c})\cdot \exp(O(\Lfx^2 T))\,.\label{eq:epsdef2_prime}
    \end{align}
    In particular, for any $\delta > 0$, if 
    \begin{align}
        \ell &\gtrsim \delta^{-1/2}(\gmax^2 + R + \ell h \Lft + \mathbb{E}\norm{\xl_0}^2)\cdot \exp(O(\Lfx^2 T)) \\
        h &\lesssim \min\bigl\{\poly(\Lg,\Lft, R, \gmax,\Lsc{*},\Lfx,\mathbb{E}\norm{\xl_0}^2)^{-1}\, \ell^{-1}\delta^{1/2}, (\delta/(\gmax^4\beta^2))^{1/2c}\bigr\} \cdot \exp(O(\Lfx^2 T)) \,, \label{eq:h_ell}
    \end{align}
    then $\epsilon'_1, \epsilon'_2 \le \delta$.
\end{lemma}

\begin{proof}
    We can bound the first term on the right-hand side of \eqref{eq:excess} using Lipschitzness of $f$ in time and space:
    \begin{equation}
        \mathbb{E}\norm{f_{T-kh}(\wtxl_{kh}) - f_{T-t}(\wtxl_t)}^2 \lesssim L^2_{f;\mathsf{x}}\,\mathbb{E}\norm{\wtxl_{kh} - \wtxl_t}^2 + h^2 L^2_{f;\mathsf{t}}\,. \label{eq:flip_term}
    \end{equation}
    For the second term on the right-hand side of \eqref{eq:excess}, we can use Lipschitzness of $g^2$ and the score:
    \begin{align}
        \MoveEqLeft \mathbb{E}\norm{g(T-t)^2\,\nabla\ln q^\leftarrow_t(\wtxl_t) - g(T-kh)^2\,\nabla\ln q^\leftarrow_{kh}(\wtxl_{kh})}^2 \\
        &\lesssim h^2 L^2_g \,\mathbb{E}\norm{\nabla \ln q^\leftarrow_{kh}(\wtxl_{kh})}^2 + g(T-t)^4\, \mathbb{E}\norm{\nabla \ln \frac{q^\leftarrow_{kh}}{q^\leftarrow_t}(\wtxl_{kh})}^2 + g(T-t)^4 \, L^2_{{\sf sc}, t}\mathbb{E}\norm{\wtxl_{kh} - \wtxl_t}^2 \\
        &\lesssim (h^2 L^2_g + g(T-t)^4 \beta^2 h^{2c})\,\mathbb{E}\norm{\nabla \ln q^\leftarrow_{kh}(\wtxl_{kh})}^2 \\
        &\quad\quad + g(T-t)^4 \beta^2 h^{2c}\,\mathbb{E}\norm{\wtxl_{kh}}^2 + g(T-t)^4\beta^2 h^{2c} + g(T-t)^4 L^2_{{\sf sc}, t}\,\mathbb{E}\norm{\wtxl_{kh} - \wtxl_t}^2 \\
        &\lesssim (h^2 \Lg^2 + \gmax^4\beta^2 h^{2c})\, \mathbb{E}\norm{\nabla \ln q^\leftarrow_{kh}(\wtxl_{kh})}^2 \\
        &\quad\quad + \gmax^4\beta^2 h^{2c}\,\mathbb{E}\norm{\wtxl_{kh}}^2 + \gmax^4\beta^2 h^{2c} + \gmax^4 \Lsc{t}^2 \, \mathbb{E}\norm{\wtxl_{kh} - \wtxl_t}^2\,. \label{eq:glip_term}
    \end{align}
    % Note that $h^2\Lft^2$ from \eqref{eq:flip_term} and $\gmax^4\Lsc{t}^2\,\mathbb{E}\norm{\wtxl_{kh} - \wtxl_t}^2$ are dominated up to constant factors by $\frac{1}{h^2}\E{\norm{\vecv_1}^2 + \cdots + \norm{\vecv_3}^2}$.
    % apart from $\Lfx^2\,\mathbb{E}\norm{\wtxl_{kh}-\wtxl_t}^2$ from \eqref{eq:flip_term} and $h^2\Lg^2 \,\mathbb{E}\norm{\nabla \ln q^\leftarrow_{kh}(\wtxl_{kh})}^2 + \gmax^4\Lsc{t}^2 \,\mathbb{E}\norm{\wtxl_{kh} - \wtxl_t}^2$ from \eqref{eq:glip_term}, all other terms appearing in these two equations are dominated up to constant factors by $\frac{1}{h^2}\E{\norm{\vecv_1}^2 + \cdots + \norm{\vecv_5}^2}$. 
    Substituting the above bounds into \eqref{eq:excess}, we get that
    \begin{align}
        \MoveEqLeft\mathbb{E}\norm{\mu_t(\wtxl_t) - \mu'_{kh}(\wtxl_{kh})}^2 \\
        &\lesssim (\Lfx^2 + \gmax^4\Lsc{t}^2)\,\mathbb{E}\norm{\wtxl_{kh} - \wtxl_t}^2 + (h^2 \Lg^2 + \gmax^4\beta^2 h^{2c})\,\mathbb{E}\norm{\nabla\ln q^\leftarrow_{kh}(\wtxl_{kh})}^2 \\
        &\quad\quad + \gmax^4\beta^2 h^{2c}\,\mathbb{E}\norm{\wtxl_{kh}}^2 + \gmax^4\beta^2 h^{2c} + \frac{1}{h^2}\,\E{\norm{\vecv_1}^2 + \cdots + \norm{\vecv_3}^2}. %\\
        % \lesssim (
        % \Lfx^2 + \gmax^4\Lsc{t}^2)\,\mathbb{E}\norm{\wtxl_{kh} - \wtxl_t}^2 + h^2\Lg^2\,\mathbb{E}\norm{\nabla \ln q^\leftarrow_{kh}(\wtxl_{kh})}^2 +
        % \frac{1}{h^2}\E{\norm{\vecv_1}^2 + \cdots \norm{\vecv_5}^2}
    \end{align}
    By applying the bounds for $\mathbb{E}\norm{\wtxl_{kh} - \wtxl_t}^2$ and $\mathbb{E}\norm{\wtxl_{kh}}^2$ in Lemma~\ref{lem:movement} and~\ref{lem:normbound} and noting that $\Lfx^2 + \gmax^4\Lsc{t}^2 \ll 1/h^2$ by \eqref{eq:hassume}, we see that the lemma follows from Lemma~\ref{lem:excess} and the definition of $\epsilon'_1, \epsilon'_2$ in Eqs.~\eqref{eq:epsdef1_prime},~\eqref{eq:epsdef2_prime}. Note that in the assumed bounds on $\ell, h$ in the lemma statement, we substituted $\mathbb{E}\norm{\xl_0}^2$ for $\mathbb{E}\norm{\wtxl_0}^2$; this is because these two quantities are identical.
\end{proof}

% \begin{align}
%     \MoveEqLeft \mathbb{E}\norm{\mu_t(\wtxl_t) - \mu'_{kh}(\wtxl_{kh})}^2 \\
%     &\lesssim (\Lfx^2 + \gmax^4 \Lsc{t}^2)\, \mathbb{E}\norm{\wtxl_{kh} - \wtxl_t}^2 \\
%     &\quad\quad + ( h^2 \Lg^2 + \gmax^4\beta^2(\ell h)^{2c} + \ell^{-2}\gmax^4 + \frac{1}{a^2}\ell^{-2} \Lfx^2 + \frac{1}{a^2}\ell^2 h^2 \Lfx^4) \, \mathbb{E}\norm{\nabla \ln q^\leftarrow_{kh}(\wtxl_{kh})}^2 \\
%     &\quad\quad + \ell^2 h^2 \Lft^2  + \ell^{-2} R^2 + \ell^2 h^2 \Lfx^2 R^2 + \gmax^4\beta^2(\ell h)^{2c} + \frac{b}{a}(\ell^{-2} \Lfx^2 + \ell^2 h^2 \Lfx^4).
% \end{align}

\subsection{Movement and norm bounds}

\begin{lemma}\label{lem:movement}
    For any integer $0 < k \le T/h$ and any $kh \le t < (k+1)h$,
    \begin{align}
        % \mathbb{E}\norm{\wtxl_t - \wtxl_{kh}}^2 \lesssim h^2\Lfx^2(R^2 + b/a) + (h^2 \gmax^4 + \frac{1}{a^2}h^2\Lfx^2)\,\mathbb{E}\norm{\nabla \ln q^\leftarrow_{kh}(\wtxl_{kh})}^2 + \E{\norm{\vecv_1}^2 + \cdots + \norm{\vecv_3}^2}\,.
        \mathbb{E}\norm{\wtxl_t - \wtxl_{kh}}^2 &\lesssim  h^2\cdot \exp(O(\Lfx^2 T)) \Bigl(\mathbb{E}\norm{\wtx_0}^2 + R^2 + \ell^2 h^2 \Lft^2 \\
        &\quad\quad  + \gmax^4 \,\max_{k\in\{0,1,\ldots,T/h\}}\,\mathbb{E}\norm{\nabla \ln q^\leftarrow_t(\wtxl_t)}^2\Bigr) + \E{\norm{\vecv_1}^2 + \cdots + \norm{\vecv_3}^2}\,.
    \end{align}
\end{lemma}

\begin{proof}
    By definition of the interpolated process, \begin{equation}
        \wtxl_t = \wtxl_{kh} - (t - kh)\,\brc{f_{T - kh}(\wtxl_{kh}) - \frac{1}{2}g(T - kh)^2\nabla \ln q^\leftarrow_{kh}(\wtxl_{kh}) + \frac{1}{h}(\vecv_1 + \cdots + \vecv_3)},
    \end{equation}
    so
    \begin{equation}
        \mathbb{E}\norm{\wtxl_t - \wtxl_{kh}}^2 \lesssim h^2 \, \mathbb{E}\norm{f_{T-kh}(\wtxl_{kh})}^2 + h^2 \gmax^4 \,\mathbb{E}\norm{\nabla \ln q^\leftarrow_{kh}(\wtxl_{kh})}^2 + \E{\norm{\vecv_1}^2 + \cdots + \norm{\vecv_3}^2}\,.
    \end{equation} 
    The proof is complete upon using Part~\ref{item:Lfx} of Assumption~\ref{assume:smooth} and Lemma~\ref{lem:normbound} to get
    % \begin{align}
    %     h^2\,\mathbb{E}\norm{f_{T-kh}(\wtxl_{kh})}^2 &\lesssim \frac{1}{a^2} h^2 \Lfx^2 \,\mathbb{E}\norm{\nabla \ln q^\leftarrow_t(\wtxl_t)}^2 + h^2 (R^2 + b/a)\,. \qedhere
    % \end{align}
    \begin{align}
        \mathbb{E}\norm{f_{T-kh}(\wtxl_{kh})}^2 &\lesssim \exp(O(\Lfx^2 T)) \Bigl(\mathbb{E}\norm{\wtx_0}^2 + R^2 + \ell^2 h^2 \Lft^2 + \gmax^4 \,\max_{k\in\{0,1,\ldots,T/h\}}\mathbb{E}\norm{\nabla \ln q^\leftarrow_t(\wtxl_t)}^2\Bigr)\,, 
    \end{align}
    where we have used that $\exp(O(\Lfx^2 T))\cdot \Lfx^2 = \exp(O(\Lfx^2 T))$.
\end{proof}

\begin{lemma}\label{lem:normbound}
    For all $0 \le t \le T$,
    \begin{equation}
        \mathbb{E}\norm{\wtxl_t}^2 \lesssim \exp(O(\Lfx^2 T)) \, \Bigl(\mathbb{E}\norm{\wtxl_0}^2  + R^2 + \ell^2 h^2 \Lft^2 + \gmax^4\,\max_{k\in\{0,1,\ldots,T/h\}}\mathbb{E}\norm{\nabla \ln q^\leftarrow_{kh}(\wtxl_{kh})}^2\Bigr)\,.
    \end{equation}
\end{lemma}

\begin{proof}
    By Lemma~\ref{lem:norm_deriv}, 
    \begin{align}
        \partial_t \, \mathbb{E}\norm{\wtxl_t}^2 &\lesssim \mathbb{E}\norm{\wtxl_t}^2 + \mathbb{E}\norm{f_{T-kh}(\wtxl_{kh})}^2 + \gmax^4\,\mathbb{E}\norm{\nabla \ln q^\leftarrow_{kh}(\wtxl_{kh})}^2 + \frac{1}{h^2}\,\E{\norm{\vecv_1}^2 + \cdots + \norm{\vecv_3}^2} \\
        &\lesssim \mathbb{E}\norm{\wtxl_t}^2 + \Lfx^2\,\mathbb{E}\norm{\wtxl_{kh}}^2 + \gmax^4\,\mathbb{E}\norm{\nabla \ln q^\leftarrow_{kh}(\wtxl_{kh})}^2 + \mathbb{E}\norm{z - \wtxl_{kh}}^2 + R^2 + \ell^2 h^2 \Lft^2 \\
        &\lesssim \mathbb{E}\norm{\wtxl_t}^2 + \Lfx^2\,\mathbb{E}\norm{\wtxl_{kh}}^2 + \gmax^4\,\mathbb{E}\norm{\nabla \ln q^\leftarrow_{kh}(\wtxl_{kh})}^2 + R^2 + \ell^2 h^2 \Lft^2\,,
    \end{align}
    where in the second step we used \eqref{eq:ugly} and $z$ is defined in \eqref{eq:zdef}, and in the third step we used \eqref{eq:zminusx} and the fact that $\ell h \ll 1$ by \eqref{eq:hassume}. By Gr\"{o}nwall applied to the interval of times $t\in[kh,(k+1)h]$ along the reverse process, we find that
    \begin{align}
        \mathbb{E}\norm{\wtxl_t}^2 &\lesssim \exp(O(h))\cdot \bigl((1 + h\Lfx^2)\,\mathbb{E}\norm{\wtxl_{kh}}^2 + h(\gmax^4\,\mathbb{E}\norm{\nabla \ln q^\leftarrow_{kh}(\wtxl_{kh})}^2 + R^2 + \ell^2 h^2 \Lft^2)\bigr) \\
        &\lesssim \exp(c\Lfx^2 h)\,\mathbb{E}\norm{\wtxl_{kh}}^2 + h\,\exp(O(h))\cdot (\gmax^4\,\mathbb{E}\norm{\nabla \ln q^\leftarrow_{kh}(\wtxl_{kh})}^2 + R^2 + \ell^2 h^2 \Lft^2)
    \end{align}
    for all $t\in [kh,(k+1)h]$ for some absolute constant $c > 0$. In particular, this bound holds for $t = (k+1)h$. Iterating this $T/h$ times, we obtain the desired bound.
\end{proof}

\noindent Recall the definition of $\Lambda,\Lambda'$ in \eqref{eq:Lamdef}.

\begin{lemma}\label{lem:integrate_FI}
    For all integers $0 \le k \le T/h$,
    \begin{align}
        \mathbb{E}\norm{\nabla \ln q^\leftarrow_{kh}(\wtxl_{kh})}^2 &\lesssim \Lambda^{O(1)}\Bigl((\Lfx + \gmax^2 \Lsc{*})d + \gmax^4\Lhigh^2 d^2 T \\
        &\quad\quad + \Lsc{*}T\max_{t\in[0,T]}\mathbb{E}\norm{\mu_t(\wtxl_t) - \mu'_{\lfloor t/h\rfloor h}(\wtxl_{\lfloor t/h\rfloor h})}^2\Bigr)
    \end{align}
\end{lemma}

\begin{proof}
    The proof follows from Lemmas~\ref{lem:part1},~\ref{lem:part2},~\ref{lem:part3}, and the bound in Lemma~\ref{lem:laplace2} with $\zeta_t \triangleq \mathbb{E}\norm{\mu_t(\wtxl_t) - \mu'_{kh}(\wtxl_{kh})}^2$ and $\zeta^2 = \int^T_0 \zeta^2_t\, \d t \le T\max_t \zeta^2_t$. Note that in the definition of $\Lambda$ and $\Lambda'$, we have a $\Lfx^2$ term in the integrand even though there is only an $\Lfx$ term in the definition of $L_t$ in Lemma~\ref{lem:part1}. The reason for this looseness is to absorb the $\exp(O(\Lfx^2 T))$ terms that appear elsewhere in the above analysis.
\end{proof}

\subsection{Putting everything together}

\begin{proof}[Proof of Theorem~\ref{thm:main}]
    Let $\delta > 0$ be a small parameter to be tuned later, and suppose $h,\ell$ satisfy~\eqref{eq:h_ell}. Then by integrating the bound in Lemma~\ref{lem:bound_drift_diff} over $0 \le t \le T$ and applying Lemma~\ref{lem:integrate_FI}, we conclude that
    \begin{align}
        \zeta^2 &\triangleq \int^T_0 \mathbb{E}\norm{\mu_t(\wtxl_t) - \mu'_{\lfloor t/h\rfloor h}(\wtxl_{\lfloor t/h\rfloor h})}^2\, \d t \\
        &\lesssim \delta T + \delta \Lambda^{O(1)}\Bigl((\Lfx + \gmax^2 \Lsc{*})dT + (1 + \gmax^2)^2\Lhigh^2 d^2 T^2 \\
        &\quad\quad + \Lsc{*}T\int^T_0  \mathbb{E}\norm{\mu_t(\wtxl_t) - \mu'_{\lfloor t/h\rfloor h}(\wtxl_{\lfloor t/h\rfloor h})}^2\,\d t\Bigr)\,.
    \end{align}
    Provided that 
    \begin{equation}
        \delta \le \frac{1}{2}\Lambda^{-O(1)}\Lsc{*}^{-1}T^{-1}\,, \label{eq:delta1}
    \end{equation}
    we can rearrange to conclude that
    \begin{equation}
        \zeta^2 %\lesssim \delta T + \delta \Lambda^{O(1)}\,\bigl((\Lfx + \gmax^2 \Lsc{*})dT + (1 + \gmax^2)^2\Lhigh^2 d^2 T^2\bigr)\, \\
        \lesssim \delta\Lambda^{O(1)}\,\bigl((\Lfx + \gmax^2\Lsc{*})dT + \gmax^4\Lhigh^2 d^2 T^2\bigr).
    \end{equation}
    By Theorem~\ref{thm:klode}, 
    \begin{align}
        \KL{\pi'_T}{\pi_T} &\lesssim (\Lambda^{O(1)} + \Lambda'^{O(1)})(L'^{1/2}_0d^{1/2} + MdT^{1/2}) \, \zeta T^{1/2} + \Lambda^{O(1)} L'^{1/2}\zeta^2 \\
        \intertext{We will take $\delta$ sufficiently small that $\zeta^2 \le 1$, in which case by upper bounding $L'_0$ by $L'$, the above is at most}
        &\lesssim (\Lambda^{O(1)} + \Lambda'^{O(1)})(L'^{1/2}d^{1/2} + MdT^{1/2}) \, \zeta T^{1/2} \\
        &\lesssim (\Lambda^{O(1)} + \Lambda'^{O(1)})\bigl((\Lfx^{1/2} + \gmax \Lsc{*}^{1/2})d^{1/2} + \gmax^2\Lhigh d T^{1/2}\bigr) \\
        &\quad\quad\times \bigl((\Lfx^{1/2} + \gmax\Lsc{*}^{1/2})d^{1/2}T^{1/2} + \gmax^2 \Lhigh dT\bigr) \delta^{1/2}T^{1/2} \\
        &\lesssim (\Lambda^{O(1)} + \Lambda'^{O(1)})\bigl((\Lfx + \gmax^2\Lsc{*})dT + \gmax^4\Lhigh^2 d^2 T^2\bigr) \delta^{1/2}T^{1/2}
    \end{align}
    We take $\delta$ so that the above is at most the target accuracy $\epsilon$. By \eqref{eq:h_ell}, this can be achieved by taking $h,\ell$ satisfying the bounds in the theorem statement.
\end{proof}

\end{document}